\documentclass{article}

\usepackage[preprint]{neurips_2019}

\usepackage[utf8]{inputenc} \usepackage[T1]{fontenc}    \usepackage{hyperref}       \usepackage{url}            \usepackage{booktabs}       \usepackage{amsfonts}       \usepackage{nicefrac}       \usepackage{microtype}      
\usepackage{amsmath}
\usepackage{amssymb}
\usepackage{amsthm}
\newtheorem{lemma}{Lemma}

\usepackage{tikz}
\usetikzlibrary{math, positioning, patterns, matrix, decorations.pathreplacing, calc}
\usepackage{pgfplots}
\usepackage{pgfplotstable}
\usepgfplotslibrary{fillbetween}

\definecolor{darkgreen}{rgb}{0.0,0.5,0.0}
\definecolor{amber}{rgb}{1.0, 0.75, 0.0}

\usepackage[caption=false,font=footnotesize]{subfig}

\usepackage{siunitx}
\usepackage{titlesec}
\usepackage{enumitem}
\usepackage{array, multirow, hhline}
\usepackage{todonotes}

\definecolor{darkgreen}{rgb}{0.0,0.5,0.0}
\definecolor{amber}{rgb}{1.0, 0.75, 0.0}

\DeclareMathOperator{\diag}{diag}
\DeclareMathOperator{\spann}{span}
\DeclareMathOperator{\nnz}{nnz}
\DeclareMathOperator{\rank}{rank}
\DeclareMathOperator{\trace}{trace}

\setcitestyle{numbers,square}

\title{Semi-Supervised Classification on Non-Sparse~Graphs
Using Low-Rank Graph~Convolutional~Networks}

\author{
	Dominik Alfke and Martin Stoll
	\\ Department of Mathematics, Chair of Scientific Computing
\\ Chemnitz University of Technology, Germany
	\\ \texttt{\{dominik.alfke, martin.stoll\}@math.tu-chemnitz.de}
}

\begin{document}

\maketitle

\begin{abstract}
Graph Convolutional Networks (GCNs) have proven to be successful tools for semi-supervised learning on graph-based datasets. For sparse graphs, linear and polynomial filter functions have yielded impressive results. For large non-sparse graphs, however, network training and evaluation becomes prohibitively expensive.
By introducing 
\emph{low-rank filters},
we gain significant runtime acceleration and simultaneously improved accuracy. We further propose an architecture change mimicking techniques from Model Order Reduction in what we call a \emph{reduced-order GCN}. Moreover, we present how our method can also be applied to \emph{hypergraph} datasets and how hypergraph convolution can be implemented efficiently.
\end{abstract}

\section{Introduction}
\label{sec:introduction}

Data science provides a multitude of approaches for making predictions within the enormous amounts of data produced in the digital age.
In settings where unsupervised learning is unreliable and supervised learning requires too much manual input, semi-supervised learning (SSL) aims at merging the best of both worlds. Effective methods may benefit from both a small set of training data and clustering information extracted from a vast amount of unlabeled data.
In this context, Graph Convolutional Networks (GCNs) are a type of neural network designed to exploit the graph-like structure present in many SSL applications.
Since their introduction \cite{bruna14}, extension, and adaptation \cite{defferard16,kipf17}, GCNs have been used to produce impressive results in various domains \cite{wu19}.

A crucial design choice in each GCN architecture is the filter function space. While the original theory allows for arbitrary functions, most practical applications employ a low-dimensional space of low-order polynomials. For sparse graphs, where each node is only connected to a small number of neighbors, these polynomials have desirable qualities that make other functions unappealing \cite{defferard16}. For dense or fully-connected graphs on the other hand, the polynomial-based convolutional operation becomes too expensive to evaluate, rendering polynomial filters impractical.
One type of such graphs with non-sparse structure are hypergraphs \cite{bretto13}, which occur naturally in many data science applications, e.g., those containing categorical data  \cite{zhou06,hein13,gao13,bosch16,purkait17}. GCNs have been employed on hypergraphs, e.g., in \cite{feng19,yadati19}.

In order to accelerate the convolutional operation with non-sparse graphs,
we employ a low-rank approach inspired by Model Order Reduction \cite{antoulas05}, a technique designed to accelerate algorithms without sacrificing too much accuracy. In addition to drastically improving training times, our low-rank architecture enables new non-polynomial filters that can even provide more accurate classification results by reinforcing simple design principles.

The rest of this paper is organized as follows. In Section~\ref{sec:gcn}, we introduce our GCN design and point out differences to existing architectures. Section~\ref{sec:hypergraph} describes how GCNs can be efficiently applied to hypergraphs. In Section~\ref{sec:experiments}, we conduct experiments that focus on the comparison of different GCN architectures, showing that low-rank and reduced-order GCNs can indeed produce better results.

\section{Semi-Supervised Learning with Low-Rank Graph Convolutional Networks}
\label{sec:gcn}

\subsection{Problem setting}
\label{sec:gcn:setting}

Suppose that a fixed dataset of $n$ points is given where each data point is associated with a $d$-dimensional feature vector. The goal is to predict an $m$-dimensional output vector for each point based on a small subset of training points for which the desired output is known. In matrix notation, we consider the full feature matrix $X \in \mathbb{R}^{n\times d}$ as input, holding the feature vectors in its rows, and we want to produce a prediction matrix $Y \in \mathbb{R}^{n \times m}$.

A neural network for this task produces a series of layer matrices $X^{(l)} \in \mathbb{R}^{n \times N_l}$ for $l=0,\ldots,L$.  Throughout the network, the number of rows in each layer matrix must not change since the $i$-th row always refers to the $i$-th data point, but the numbers of columns $N_0,\ldots,N_L$ can vary. We first set $X^{(0)}=X$ and then iteratively compute the next layer matrix by applying a parametric layer operation to the previous one. The last layer matrix $X^{(L)}$ is the network output which can be used to produce the matrix of predicted values for each data point, $Y = \Psi(X^{(L)})$, by applying a prediction function $\Psi : \mathbb{R}^{N_L} \to \mathbb{R}^m$ to each row. $X^{(1)}$ through $X^{(L-1)}$ are also called \emph{hidden} layer matrices. 

One popular application of this structure is data point classification, where each data point is to be assigned to one of $m$ classes. This is achieved by choosing $\Psi$ as the softmax function.

\subsection{Graph convolutional networks}
\label{sec:gcn:gcn}

In some semi-supervised learning applications, the dataset may come with additional information on special connections between data points. This adjacency information describes a graph and is exploited in Graph Convolutional Networks (GCN).
There, the layer operation performs a \emph{graph convolution} on the layer matrices \cite{shuman13,bruna14,kipf17}. The network learns spectral filter functions from a predefined function space to determine how the next layer matrix depends on the previous one.

Like many methods from graph theory, GCN rely on the graph Laplacian matrix to describe the spectral properties of a graph \cite{luxburg07}. It is defined using the weighted adjacency matrix $W$, i.e., the matrix where $W_{ij}$ holds 
the strength of the connection between the $i$-th and $j$-th node, where $W_{ij}=0$ represents no connection.
The degree matrix $D$ is the diagonal matrix where $D_{ii}$ holds the sum of the weights of all edges connected to the $i$-th node. Based on these and the identity matrix $I$, the normalized variant of the graph Laplacian operator is defined as
\begin{equation}
\label{eq:glo}
\mathcal{L} = I - D^{-1/2} W D^{-1/2}.
\end{equation}
Let $\mathcal{L} = U \Lambda U^T$ be an eigenvalue decomposition of this matrix,
i.e., $\Lambda$ is a diagonal matrix holding the eigenvalues $\lambda_1,\ldots,\lambda_n$, and $U$ holds the corresponding eigenvectors in its columns.

Suppose that $K$ real-valued functions $\varphi_1, \ldots, \varphi_K$ are given, spanning a $K$-dimensional spectral filter space. 
Since the eigenvalues of $\mathcal{L}$ are restricted to $[0,2)$ \cite{luxburg07}, the functions only need to be defined on that interval.
Then the GCN layer operation \cite{kipf17} is defined as
\begin{equation} \label{eq:gcn_full_op}
	X^{(l)} = \sigma_{l}\left(\sum_{k=1}^K \mathcal{K}^{(k)} X^{(l-1)} \Theta^{(k,l)}\right) 
	\quad \text{with} \quad 
\mathcal{K}^{(k)} = U \varphi_k(\Lambda) U^T.
\end{equation}
Here, $\varphi_k(\Lambda)$ denotes the diagonal matrix holding the functions values of $\varphi_k$ in the diagonal elements of $\Lambda$.
The \emph{kernel matrices} $\mathcal{K}^{(k)} \in \mathbb{R}^{n\times n}$ only depend on the graph and the filter function space. The \emph{weight matrices} $\Theta^{(k,l)} \in \mathbb{R}^{N_{l-1}\times N_{l}}$ are the network parameters to be determined in training. $\sigma_{l} : \mathbb{R} \to \mathbb{R}$ are the activation functions of each layer.

\subsection{Choice of the spectral filter space}
\label{sec:gcn:basis}

When designing a GCN, arguably the most important parameter is the number and shape of the basis functions for
$\spann\{ \varphi_1,\ldots,\varphi_K \}$,
the space of permitted spectral filter functions. While the choice of the dimension $K$ mainly comes down to a trade-off between flexibility and efficiency (where $K=1$ is desirable if a single basis function already incorporates the desired behaviour), we can formulate multiple design principles to be observed in the choice of the $\varphi_k$:
\begin{enumerate}[leftmargin=1.5cm]
\item[(DP1)] Efficient setup of $\mathcal{K}^{(k)} = U \varphi_k(\Lambda) U^T$, preferably without actually having to compute the full eigenvalue decomposition of $\mathcal{L}$.
\item[(DP2)] Efficient evaluation of matrix products $\mathcal{K}^{(k)} X$. 
\item[(DP3)] $|\varphi_k(\lambda)|$ is large if $\lambda$ is small but nonzero, and small if $\lambda$ is large.
\end{enumerate}
The last goal is based on the observation that premultiplication with $\mathcal{K}^{(k)}$ amplifies vectors in the direction of an eigenvector $u$ by factor $|\varphi_k(\lambda)|$, where $\lambda$ is the corresponding eigenvalue. Graph theory tells us that an eigenvector $u$ contains clustering information if its eigenvalue $\lambda$ is small but nonzero, or noise if $\lambda$ is large \cite{luxburg07}. Here, clustering means that the $i$-th and $j$-th components of the vector are similar if and only if nodes $i$ and $j$ have a strong connection in the dataset graph. Semi-supervised learning applications naturally call for 
amplification of such clustering vectors and damping of noise.
This can be achieved by following (DP3).

Note that these design principles only give a preliminary heuristic on what constitutes a ``good'' basis function. The choice leading to the best final results may very much depend on the specific dataset and can hardly be predicted a priori.

\subsubsection{Polynomial basis functions for full-rank kernel matrices}
\label{sec:gcn:basis:fullrank}

The first design principle can be observed most easily by using polynomial basis functions, which avoid eigenvalue computation altogether. A simple example following (DP3) can be $\varphi(\lambda) = 1 - \frac{\lambda}{\lambda_n}$, where the largest eigenvalue $\lambda_n$ can be computed numerically or simply estimated as a suitable value between 1 and 2. The resulting kernel matrix is $\mathcal{K}= I - \frac{1}{\lambda_n} \mathcal{L}$. 

Affine linear functions are especially suitable if $\mathcal{L}$ is sparse, since its sparsity pattern will be maintained in $\mathcal{K}$. This is in line with (DP2). Higher-order polynomials act similarly, but increase the number of non-zeros. 
Generally, it is possible to choose $\varphi_1,\ldots,\varphi_K$ as a basis for the function space of polynomials with degree up to $K-1$. With sparse graphs, this offers a straightforward interpretation of $K$ as a localization parameter, as the $i$-th entry $\mathcal{K}_k x$ only depends on the values of $x$ in nodes connected to $i$ by a path of length $K$ \cite{defferard16}.

For non-sparse graphs, however, the resulting full-rank kernel matrix may cause expensive layer operations, violating (DP2). In that case, polynomial filters are intrinsically less appealing.

\subsubsection{Low-rank basis functions for dense Laplacians}
\label{sec:gcn:basis:lowrank}

As opposed to polynomials, we would like to introduce the basis function
\begin{equation} \label{eq:pseudoinverse}
\varphi(\lambda) = \begin{cases} 0 & \text{if } \lambda = 0, \\ \frac{1}{\lambda} & \text{if } \lambda > 0, \end{cases}
\end{equation}
producing the kernel matrix $\mathcal{K} = \mathcal{L}^\dagger$, which denotes the Moore-Penrose pseudoinverse of $\mathcal{L}$ \cite{golubvanloan}. While this choice fits the third design principle particularly well, it violates the two others because we would either have to compute the explicit (non-sparse) pseudoinverse beforehand or solve linear systems of equations in each network evaluation and training step, which is far too expensive.

One way to overcome these problems is to replace $\mathcal{K}$ by its best low-rank approximation.
This is a basic technique in Numerical Linear Algebra \cite{golubvanloan} and Model Order Reduction \cite{antoulas05}, but in the convolutional setting they can be motivated in a particularly straightforward way.
Consider an arbitrary basis function $\varphi$. Given a target rank $r$, we can define a second basis function
\begin{equation*}
\tilde{\varphi}(\lambda) = \begin{cases} \varphi(\lambda) & \text{if $\lambda$ is one of the $r$ eigenvalues 
for which $|\varphi(\lambda_i)|$ is largest,
} \\ 0 & \text{else.} \end{cases}
\end{equation*}
This means that $\tilde{\varphi}$ is equal to $\varphi$ on the $r$ ``dominant'' eigenvalues (with respect to $\varphi$), but $0$ on all others. As a result, we see that the kernel matrix of $\tilde{\varphi}$ is
\[
	\tilde{\mathcal{K}} = U \tilde{\varphi}(\Lambda) U^T = U_r \tilde{\varphi}(\Lambda_r) U_r^T = U_r \varphi(\Lambda_r) U_r^T \approx \mathcal{K},
\]
where $\Lambda_r \in \mathbb{R}^{r\times r}$ denotes the diagonal matrix of only these $r$ dominant eigenvalues, and $U_r \in \mathbb{R}^{n \times r}$ holds the corresponding eigenvectors in its columns. 
It can be shown that $\tilde{\mathcal{K}}$ is indeed the best rank-$r$ approximation to $\mathcal{K}$, cf. Lemma~1 in the appendix.

Using low-rank kernel matrices has three major upsides.
First, consider a non-sparse $\mathcal{L}$ such that polynomial basis functions lead to inefficient $\mathcal{K}$, violating (DP2). Then matrix products with the approximation $\tilde{\mathcal{K}}$ are much cheaper to evaluate, reducing the asymptotic number of multiplications from $n^2$ to $2nr$ per column.
Second, consider a non-polynomial basis function like the pseudo-inverse from \eqref{eq:pseudoinverse}. Setting up $\mathcal{K}$ requires a full eigenvalue decomposition of $\mathcal{L}$, violating (DP1). Setting up $\tilde{\mathcal{K}}$ on the other hand requires only a small number of eigenpairs, e.g., the smallest $r$ eigenvalues (depending on the shape of $\varphi$). This task can be solved with efficient numerical methods, making these non-polynomial functions feasible in the first place. Third, if $\varphi$ already follows (DP3), the dominant eigenvalues are those with clustering information. Setting $\tilde{\varphi}(\lambda) = 0$ for all non-clustering eigenvalues reinforces this design principle, damping noise to zero.

If all filter basis functions $\varphi_1,\ldots,\varphi_K$ of a GCN originate in low-rank approximations with a shared set of $r$ dominant eigenvalues, we will call the network a \emph{low-rank GCN} of rank $r$. Using the notation with $\Lambda_r$ and $U_r$ as before, the convolutional layer operation from \eqref{eq:gcn_full_op} can be written as
\begin{equation} \label{eq:gcn_lowrank_op}
X^{(l)} = \sigma_l \left( U_r \sum_{k=1}^K \varphi_k(\Lambda_r) U_r^T X^{(l-1)} \Theta^{(l,k)} \right).
\end{equation}

\subsection{Spectral activation and reduced-order networks}
\label{sec:gcn:reduced}

In the context of low-rank GCNs from the previous section, we introduce a small but potentially powerful architecture change that handles activation in the spectral instead of the spatial domain. We achieve that by applying $\sigma_l$ \emph{before} premultiplication with $U_r$ in \eqref{eq:gcn_lowrank_op}. The resulting layer operation is
\begin{equation} \label{eq:layer_op_spectral_activation}
X^{(l)} =  U_r \sigma_l\left( \sum_{k=1}^K\varphi_k(\Lambda_r) U_r^T X^{(l-1)} \Theta^{(l, k)} \right).
\end{equation}
On its own, this change may look arbitrary. However, since $\mathcal{L}$ is symmetric, the eigenvectors are orthogonal to each other, i.e., $U_r^T U_r = I$. Consider two sequential convolutional layers. Instead of multiplying with $U_r$ as the last step of the first layer and immediately multiplying with $U_r^T$ at the beginning of the second, we can now avoid these multiplications altogether.

This results in a network structure where we never have to compute the layer matrices, but instead work with their spectral transformations $\hat{X}^{(l)} = U_r^T X^{(l)} \in \mathbb{R}^{r \times N_l}$. After an initial \emph{reduction} step to compute $\hat{X}^{(0)}$, all layer operations are now significantly more efficient. The original output matrix can be retrieved through a final \emph{projection} step, $X^{(L)} = U_r \hat{X}^{(L)}$. This approach is reminiscent of classical methods used in Model Order Reduction for ODEs \cite{antoulas05}. The resulting network structure is depicted in Figure~\ref{fig:reduced_GCN}. We will refer to this type of network as \emph{reduced-order graph convolutional networks}.

\begin{figure}
\begin{center}
	\begin{tikzpicture}
	
	\matrix (m) [matrix of math nodes, column sep=-2pt, ampersand replacement=\&] {
		\mathbb{R}^{n \times N_0} \& \to \& \mathbb{R}^{r \times N_0} \& \to \& \cdots \& \to \& \mathbb{R}^{r \times N_{l-1}} \& \to \& \mathbb{R}^{r \times N_{l}} \& \to \& \cdots \& \to \& \mathbb{R}^{r \times N_L} \& \to \& \mathbb{R}^{n \times N_L} \\
		X^{(0)} \& \mapsto \& \hat{X}^{(0)} \& \mapsto \& \cdots \& \mapsto \& \hat{X}^{(l-1)} \& \mapsto \& \hat{X}^{(l)} \& \mapsto \& \cdots \& \mapsto \& \hat{X}^{(L)} \& \mapsto \& X^{(L)} \\
	};
	
	\node[above=0cm of m-1-1] {Input};
	\node[above=0cm of m-1-15] {Output};
	
	\draw[decorate, decoration={brace, amplitude=7pt}] (m-2-9.south east) -- (m-2-7.south west) node[midway, below=7pt] (convdesc) {Convolution (for $l=1,\ldots,L$):};
	
	\node[below=0cm of convdesc, draw=black, rounded corners] {$\displaystyle \hat{X}^{(l)} = \sigma_l\left( \sum_{k=1}^{K} \varphi_k(\Lambda_r) \hat{X}^{(l-1)} \Theta^{(l,k)} \right)$};
	
	\draw[decorate, decoration={brace, amplitude=7pt}] (m-2-3.south east) -- (m-2-1.south west) node[midway, below=7pt] (inputdesc) {Reduction:};
	
	\node[below=0cm of inputdesc, draw=black, rounded corners] {$\displaystyle \hat{X}^{(0)} = U_r^T X^{(0)}$};
	
	\draw[decorate, decoration={brace, amplitude=7pt}] (m-2-15.south east) -- (m-2-13.south west) node[midway, below=7pt] (outputdesc) {Projection:};
	
	\node[below=0cm of outputdesc, draw=black, rounded corners] {$\displaystyle X^{(L)} = U_r \hat{X}^{(L)}$};
	
	\def\dx{0.05}
	\def\dy{0.05}
	\def\n{1.5}
	\def\r{0.5}
	
	\path (m-1-1.base) +(0, 2cm) coordinate (sketchcenter0);
	\path (m-1-3.base) +(0, 2cm) coordinate (sketchcenter1);
	\path (m-1-13.base) +(0, 2cm) coordinate (sketchcenter6);
	\path (m-1-15.base) +(0, 2cm) coordinate (sketchcenter7);
	\foreach \i [count=\j from 1] in {2,...,5} {
		\coordinate (sketchcenter\i) at ($(sketchcenter1)!\j/5.0!(sketchcenter6)$);
	}

	\foreach \Nl/\y [count=\i from 0] in {12/\n, 12/\r, 8/\r, 6/\r, 4/\r, 3/\r, 2/\r, 2/\n} {
		\pgfmathtruncatemacro\NlMinusOne{\Nl-1}
		\foreach \j in {1,...,\NlMinusOne} {
			\pgfmathsetmacro\z{\j - 0.5*\Nl}
			\draw[black!50] (sketchcenter\i)
			++(\z*\dx, 0.5*\y + \z*\dy)
			-- ++(0, -\y);
		}
		\draw[black] (sketchcenter\i) 
		++(0.5*\dx*\Nl, 0.5*\y + 0.5*\dy*\Nl) 
		-- ++(0, -\y) 
		-- ++(-\dx*\Nl, -\dy*\Nl)
		-- ++(0, \y)
		-- cycle;
	}
	
	\foreach \j [count=\i from 1] in {2,...,6} {
		\draw[->, red, thick, shorten >=0.3cm] (sketchcenter\i) -- (sketchcenter\j);
	}
	\draw[->, blue!60, thick, shorten >=0.3cm] (sketchcenter0) -- (sketchcenter1);
	\draw[->, blue!60, thick, shorten >=0.3cm] (sketchcenter6) -- (sketchcenter7);

	\end{tikzpicture}
\end{center}
\caption{Layers of a reduced-order GCN.}
\label{fig:reduced_GCN}
\end{figure}
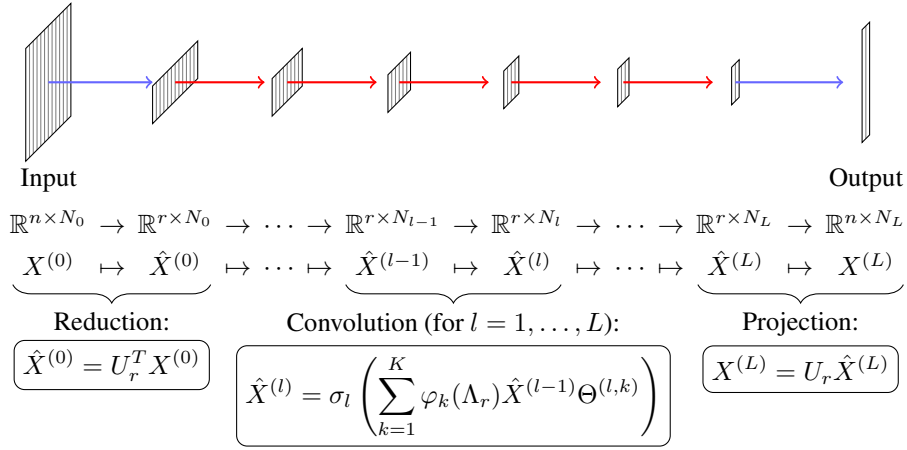

\section{Efficient techniques for GCNs on hypergraphs}
\label{sec:hypergraph}

Hypergraphs can be viewed as a natural generalization of classical graphs \cite{bretto13}. Since the definition of graph convolution only requires the graph Laplacian, GCNs can be generalized to hypergraphs by just specifying a hypergraph Laplacian. This has been proposed in competing ways. Feng et al. \cite{feng19} use the Laplacian definition from \cite{zhou06}, focusing on hypergraph generation instead of additional efficiency gains.
Yadati et al. \cite{yadati19} use evolving matrix representations of a non-linear Laplacian operator. We will follow the former method due to its better compatibility with low-rank filters. 
Where a graph has edges that connect exactly two nodes with each other, a hypergraph has \emph{hyperedges} that represent a connection between any number of nodes. These hyperedges are simply node sets $e \subset \{1,\ldots,n\}$ associated with an arbitrary weight $w_e > 0$. The set $E$ of all hyperedges can be given by its incidence matrix $H \in \mathbb{R}^{n \times |E|}$ with entries $h_{ie} = 1$ if node $i$ is a part of hyperedge $e$, and $h_{ie} = 0$ otherwise. Node degrees can be defined the same way as for classical graphs via $d_i = \sum_{e \in E} h_{ie} w_e$, i.e., the node degree $d_i$ is the sum of weights of all hyperedges that contain $i$. In addition, hypergraphs call for edge degrees that are just defined as the number of nodes contained in a hyperedge, $|e| = \sum_{i = 1}^n h_{ie}$.
We will use $D_V$, $D_E$, and $W_E$ to denote the diagonal matrices holding the node degrees $d_i$, edge degrees $|e|$, and edge weights $w_e$, respectively.

The Laplacian definition introduced in \cite{zhou06} and used in many learning applications \cite{gao13,bosch16} is given by
\begin{equation}
\label{eq:hypergraph:laplacian}
\mathcal{L}= I - D_V^{-1/2} H W_E D_E^{-1} H^T D_V^{-1/2}.
\end{equation}
The $i$-$j$-th entry of $\mathcal{L}$ is nonzero if nodes $i$ and $j$ share membership in at least one hyperedge.
In many applications this will be the case for most pairs of nodes. Hence the hypergraph Laplacian is generally not sparse.
Note that because $D_V^{-1/2} H W_E D_E^{-1} H^T D_V^{-1/2}$ is positive semi-definite, the eigenvalues of $\mathcal{L}$ are restricted to the interval $[0,1]$. 
If $|E| < n$, the largest eigenvalue is 1 with a multiplicity of $n-|E|$.
This knowledge may play a role in the design of filter functions, see Section~\ref{sec:gcn:basis}.

In some applications, the number of hyperedges is as large as $n$ or larger. This may especially be the case for some automatically generated hypergraphs, e.g., as in \cite{feng19}. In that case, the hypergraph Laplacian is just a large dense matrix without any exploitable structure, leading to expensive convolution. Just as with classical graphs, low-rank filters and reduced-order networks can be used to gain huge runtime boosts.

\subsection{Hypergraphs with a small number of hyperedges}
\label{sec:hypergraph:structure}

In the special case that the number of hyperedges is significantly smaller than the number of nodes, $|E| \ll n$, the Laplacian definition directly exhibits a useful structure.
This occurs often if the hyperedges come from real-world categorical properties. 
For automatically generated hypergraphs, there is precedence for the benefits of generating fewer, larger hyperedges \cite{purkait17}.

In this case, the matrix subtracted from the identity in \eqref{eq:hypergraph:laplacian} has rank $|E|$, so $\mathcal{L}$ is a linear combination of the identity and a low-rank matrix. 
The same structure is also exhibited by the kernel matrix $\mathcal{K}$ for any arbitrary filter function $\varphi$. For polynomial filters, the concrete factors in this structure can be computed explicitly from the hypergraph data, while for other filters, a singular value decomposition is required in a precomputation step. Explicit formulas for the structured setup of $\mathcal{K}$ are given in the appendix. In both cases, exploiting the Laplacian structure makes the network setup and training much more efficient.
The full matrix setup can be avoided altogether and the asymptotic number of scalar multiplications required for each column of a matrix product $\mathcal{K} X$ is reduced from $n^2$ to $2n|E|$.

Due to this structure, full-rank non-polynomial filters like the pseudoinverse function \eqref{eq:pseudoinverse} can now be employed efficiently. The runtime advantages of low-rank filters are thus less drastic than for general non-sparse graphs, but they are still apparent if the target rank is smaller than $|E|$.

\section{Experimental results}
\label{sec:experiments}

We will now present the performance of our proposed methods by comparing different GCN architectures. The goal of these experiments is to investigate to what degree low-rank filters can outperform traditional filters and how much the theoretical benefits of the low-rank approach come into play in practice.

All results are produced with code available online\footnote{
		\url{https://www.github.com/dominikalfke/LowRankGCNExperiments}}
based on our \textsc{Julia} implementation,\footnote{
		\url{https://www.github.com/dominikalfke/GCNModel}}
which uses TensorFlow\footnote{\url{https://www.tensorflow.org} and \url{https://www.github.com/malmaud/TensorFlow.jl}} for network training. In all experiments we employ the general architecture from \cite{kipf17} with one hidden layer ($L=2$), a one-dimensional filter function space ($K=1$), the ReLU function for $\sigma_1$, and the identity function for $\sigma_2$. Further details on the implementation, datasets, architecture, parameters, and training process are given in the appendix.

We will compare three different general architectures: The traditional \textbf{GCN} with full-rank filters, its equivalent with low-rank filters as in Section~\ref{sec:gcn:basis:lowrank} (\textbf{L-GCN}), and reduced-order GCN with spectral activation as in Section~\ref{sec:gcn:reduced} (\textbf{R-GCN}). Within these, we will employ three different filter basis functions (or their low-rank approximations, depending on the architecture):  \textbf{Linear} ($\varphi(\lambda) = 1 - \lambda$), \textbf{Quadratic} ($\varphi(\lambda) = (\lambda - 1)^2$), and \textbf{Pseudoinverse} (as in \eqref{eq:pseudoinverse}). All of these follow the third design principle. 
The linear GCN architecture is the one used most often with sparse graphs \cite{kipf17,wu19} and in the hypergraph setup from \cite{feng19}. The only difference is the \emph{re-normalization trick} from \cite{kipf17}, which we do not employ here because we obtained better results without it. 

\subsection{Spiral dataset}

One example of non-sparse graphs are those that do not come directly from real-world data, but are constructed artificially based on node feature vectors $x_i \in \mathbb{R}^d$ and a Gaussian kernel function. The resulting loop-free fully connected graph has the adjacency matrix
\[
W_{ij} = \begin{cases} \exp\left(\frac{-\|x_i-x_j\|^2}{\sigma^2}\right) & \text{if $i\neq j$,} \\ 0 & \text{else,} \end{cases}
\]
with a parameter $\sigma > 0$. We can approximate the smallest eigenvalues of the resulting graph Laplacian using the method presented in \cite{alfke18}. We will use a simple dataset from that paper consisting of $n=10\ 000$ three-dimensional feature vectors from five overlapping orbs. We set $N_0=3$, $N_1=4$, and $N_2=5$. Results are presented in Table~\ref{tbl:spiral}.

\begin{table}
\caption{Results for spiral dataset}
\label{tbl:spiral}
\centering
\begin{tabular}{@{}llccc@{}} \toprule
	\multirow{2}{*}{Network} & \multirow{2}{*}{Filter function} & \multicolumn{2}{c}{Time} & \multirow{2}{*}{Accuracy} \\ \cmidrule(lr){3-4}
	& & Setup & Training & \\ \midrule
	
	\multirow{2}{*}{GCN} 
	& Linear & 3.8 s  & 2087 s & 78.54 \% \\
	& Quadratic & 20.3 s & 1396 s & 73.39 \% \\ 
	\midrule
	
	\multirow{2}{4em}{L-GCN (rank 10)} 
	& Linear & 0.39 s & 3.32 s & 76.09 \% \\ 
	& Quadratic & 0.39 s & 3.39 s & 69.27 \% \\ 
	& \textbf{Pseudoinverse} & 0.44 s & 3.33 s & \textbf{92.23 \%} \\ 
	\midrule 
	\multirow{2}{4em}{R-GCN (rank 10)}
	& Linear & 0.39 s & 2.36 s & 46.02 \% \\
	& Quadratic & 0.39 s & 2.41 s & 32.69 \% \\ 
	& Pseudoinverse & 0.44 s & 2.37 s & 55.41 \% \\ 
	\bottomrule 
\end{tabular}
\end{table}

For the low-rank and reduced-order networks, we used the \textsc{Matlab} implementation published with \cite{alfke18} to precompute the required eigeninformation efficiently without setting up the full matrix. Without this method, the low-rank setup times in Table~\ref{tbl:spiral} would be approximately 4.3 seconds instead.

\subsection{Hypergraph datasets}

We will now introduce two hypergraph datasets obtained from the UCI Machine Learning Repository. In both cases, each node comes with a few discrete properties. For each property, we create hyperedges connecting all nodes with the same property value. The resulting hypergraph is used for convolution. Twenty nodes are marked as training nodes. We use the hypergraph incidence matrix as the network input $X^{(0)}$. The layer widths are fixed by $N_0=|E|$ and $N_2=m$, while the hidden layer width $N_1$ is chosen freely. For the full-rank GCN architecture, we add a comparison of a \textit{naive} implementation using the full Laplacian matrix and an \textit{efficient} one utilizing the structure discussed in Section~\ref{sec:hypergraph:structure}.

The \emph{Cars Evaluation} dataset\footnote{\url{https://archive.ics.uci.edu/ml/datasets/Car+Evaluation}} consists of $n=1728$ theoretical car setups which are to be categorized into $m=4$ classes. Six properties are turned into $|E|=21$ hyperedges and we choose a hidden layer width of $N_1=8$.
The \emph{Mushroom} dataset\footnote{\url{https://archive.ics.uci.edu/ml/datasets/Mushroom}} consists of $n=8124$ mushrooms samples which are to be predicted to be edible or poisonous, i.e., $m=2$. We obtain $|E|=112$ hyperedges and choose $N_1=16$.
Results for these datasets are listed in Table~\ref{tbl:hypergraphs}. Runtimes for setup and training are aggregated.

\begin{table}
\caption{Results for hypergraph datasets}
\label{tbl:hypergraphs}
\centering
\begin{tabular}{@{}llcccc@{}} 
	\toprule
	\multirow{2}{*}{Network} & \multirow{2}{*}{Filter function} & \multicolumn{2}{c}{Cars} & \multicolumn{2}{c}{Mushrooms} \\ 
	\cmidrule(lr){3-4} \cmidrule(lr){5-6}
	& & Time & Accuracy & Time & Accuracy \\
	\midrule
	
GCN & Linear & 35.3 s & 63.04 \% & 1372 s & 88.82 \% \\ 
	(naive) & Quadratic & 25.4 s & 27.39 \% & 945 s & 70.09 \% \\ 
	\midrule
	
	\multirow{3}{4em}{GCN (efficient)} 
	& Linear & 1.40 s & 63.04 \% & 16.73 s & 88.82 \% \\ 
	& Quadratic & 1.39 s & 27.39 \% & 16.88 s & 70.09 \% \\ 
	& \textbf{Pseudoinverse} & 1.42 s & \textbf{93.44 \%}  & 16.84 s & \textbf{91.76 \%} \\ 
	\midrule
	
	\multirow{3}{4em}{L-GCN (rank 20)}
	& Linear & 1.39 s & 63.03 \% & 7.36 s &  89.14 \% \\
	& Quadratic & 1.40 s & 27.39 \% & 7.34 s & 53.61 \% \\ 
	& \textbf{Pseudoinverse} & 1.42 s & \textbf{98.90 \%} & 7.31 s & \textbf{91.72} \% \\ 
	\midrule
	
	\multirow{3}{4em}{R-GCN (rank 20)} 
	& Linear & 1.00 s & 63.16 \% & 4.54 s & 87.90 \% \\
	& Quadratic & 0.99 s & 27.04 \% & 4.47 s & 81.05 \%  \\ 
	& \textbf{Pseudoinverse} & 0.99 s & \textbf{90.33 \%} & 4.49 s & \textbf{92.83} \% \\
	\bottomrule 
\end{tabular} 
\end{table}

\subsection{Performance comparison}

In all our experiments, the pseudoinverse filter function produces the best results, in many cases by a large margin. The linear function only comes close in the mushrooms example, while the quadratic function fails completely most of the time.

The traditional GCN networks using the full Laplacian take a long time to train, which will be prohibitive in certain applications. It is unclear at this point why the quadratic kernel leads to shorter training times. Exploiting the hypergraph Laplacian structure achieves a speed-up of up to 2500\% (cars) and 8400\% (mushrooms), preserving the accuracy.

An even greater runtime gain is achieved by low-rank filters. In almost all cases, these also yields better accuracy results than their full-rank counterparts. We interpret this remarkable observation to be caused by the fact that these filters rigorously ban all noise eigenvectors from the output, reinforcing (DP3). With the cars dataset, there is no speed-up over the efficient full-rank implementation because these kernels already have rank $|E|=21$. The larger $|E|$, the more significant the acceleration becomes since we expect that a target rank of 20 to 30 should suffice even for larger datasets.

The reduced-order architecture, however, produces very polarizing results. For the spiral and cars datasets, the accuracy strongly decreases compared to the simple low-rank approach. For the mushrooms dataset, on the other hand, the accuracy is even improved, producing the best result for that dataset both from a runtime and accuracy perspective. It appears that the success of activation in the spectral domain following \eqref{eq:layer_op_spectral_activation} depends strongly on the data. When testing a new dataset, there is currently no way to predict a priori whether an R-GCN will succeed.

For comparison, the GCN method using the nonlinear hypergraph Laplacian has a reported accuracy of 90\% for the mushrooms dataset with the same number of training nodes \cite[Figure~5]{yadati19}.

\section{Conclusion}

The main contribution of our paper is the introduction of low-rank filters. These do not only decrease runtimes by several orders of magnitude over full-rank dense kernels, they also produce more accurate classifications. We further introduce the pseudoinverse filter, which seems to be intrinsically a better choice than the current standard linear filter. Moreover, we present the reduced-order GCN architecture, which is strongly dependent on the dataset but in certain cases further improves runtime and accuracy.

\titleformat*{\section}{\normalfont\LARGE\bfseries\filcenter}
\section*{Appendix}
\setcounter{section}{0}
\renewcommand{\thesection}{A\arabic{section}}
\titleformat{\section}{\normalfont\large\bfseries}{\thesection}{1em}{}

\section{Overview}

In Section~\ref{sec:app:gcn}, we give more details on our problem setting, the general GCN architecture, and properties of low-rank filters. Most notably, we present a new, more general way of introducing a \emph{smoothed graph Laplacian} in Section~\ref{sec:app:gcn:smooth}, which may be theoretically interesting due to its connection to existing techniques, even though it does not necessarily produce better results.

Section~\ref{sec:app:hypergraph} covers details concerning the hypergraph Laplacian. Most notably, we give formulas for efficient evaluation of the convolutional layer operation in Section~\ref{sec:app:hypergraph:efficient}, and in Section~\ref{sec:app:hypergraph:interpretation} we incorporate the hypergraph Laplacian into our previously introduced smoothing framework.

Finally in Section~\ref{sec:app:experiments}, we list a multitude of details about the datasets and implementation, targeted both at efficiency and reproducibility. We also give additional results for a rank comparison and for tests with our smoothing framework.

\section{Graph Convolutional Networks}
\label{sec:app:gcn}

\subsection{Supervised learning vs. semi-supervised learning}
\label{sec:app:gcn:ssl}

Most neural network applications work in the setting of \emph{supervised learning}. There, a dataset of training points is given, all of which are associated with a $d$-dimensional feature vector and an $m$-dimensional desired output vector. The task is then to find a neural network that can yield a predicted output $y \in \mathbb{R}^m$ for any new data point feature vector $x \in \mathbb{R}^d$.

Semi-supervised learning, which is considered in this paper, differs from that viewpoint. Here, a fixed dataset of $n$ data points is given to which no new data points will be added. Only a few of these data points are labelled. Instead of looking for a mapping $\mathbb{R}^d \to \mathbb{R}^m$, we understand the neural network as a mapping of the full feature matrix $X \in \mathbb{R}^{n \times d}$ to the complete output matrix $Y \in \mathbb{R}^{n \times m}$. This way, we restrict ourselves from evaluating the network for single feature vectors, especially future ones. Moreover the effort of each evaluation or training step depends linearly on the size $n$ of the dataset. In recompense for these drawbacks, 
semi-supervised learning allows the network to detect clusterings in the dataset,
giving good results already for a very small number of training points.

Figure~\ref{fig:nn_ssl} depicts the series of operations in a general neural network for SSL.

\begin{figure}
	\begin{center}
		\begin{tikzpicture}
		
		\matrix (m) [matrix of math nodes, column sep=-2pt, ampersand replacement=\&] {
			\mathbb{R}^{n \times N_0} \& \to \& \cdots \& \to \& \mathbb{R}^{n \times N_{l}} \& \to \& \mathbb{R}^{n \times N_{l+1}} \& \to \& \cdots \& \to \& \mathbb{R}^{n \times N_L} \& \to \& \mathbb{R}^{n \times m} \\
			X^{(0)} \& \mapsto \& \cdots \& \mapsto \& X^{(l)} \& \mapsto \& X^{(l+1)} \& \mapsto \& \cdots \& \mapsto \& X^{(L)} \& \mapsto \& Y \\
		};
		
		\node[above=0.2cm of m-1-1, anchor=base] (inputlabel) {\small Input};
		\node[above=0.2cm of m-1-11, anchor=base] {\small Output};
		\node[above=0.2cm of m-1-13, anchor=base] {\small Prediction};
		
		\draw[decorate, decoration={brace, amplitude=7pt}] (m-2-7.south east) -- (m-2-5.south west) node[midway, below=7pt] (convdesc) {\small $l$-th layer operation};
		
		\draw[decorate, decoration={brace, amplitude=7pt}] (m-2-13.south east) -- (m-2-11.south west) node[midway, below=7pt] (outputdesc) {\small $Y = \Psi(X^{(L)})$};

		\def\dx{0.05}
		\def\dy{0.05}
		\def\n{1.5}
		
		\path (m-1-1.base) +(0, 2cm) coordinate (sketchcenter0);
		\path (m-1-11.base) +(0, 2cm) coordinate (sketchcenter5);
\foreach \i in {1,...,4} {
			\coordinate (sketchcenter\i) at ($(sketchcenter0)!\i/5.0!(sketchcenter5)$);
		}

		\foreach \Nl [count=\i from 0] in {12, 8, 6, 4, 3, 2} {
			\pgfmathtruncatemacro\NlMinusOne{\Nl-1}
			\foreach \j in {1,...,\NlMinusOne} {
				\pgfmathsetmacro\z{\j - 0.5*\Nl}
				\draw[black!50] (sketchcenter\i)
				++(\z*\dx, 0.5*\n + \z*\dy)
				-- ++(0, -\n);
			}
\draw[black] (sketchcenter\i) 
			++(0.5*\dx*\Nl, 0.5*\n + 0.5*\dy*\Nl) 
			-- ++(0, -\n) 
			-- ++(-\dx*\Nl, -\dy*\Nl)
			-- ++(0, \n)
			-- cycle;
		}
		
		\foreach \j [count=\i from 0] in {1,...,5} {
			\draw[->, red, thick, shorten >=0.3cm] (sketchcenter\i) -- (sketchcenter\j);
		}
		
		\def\h{0.07}
		
		\foreach \y [count=\i] in {1.0, 0.8, ..., -1.0} {
			\path (m-1-1.base) +(-2, 2+\y) coordinate (inputcenter\i);
			\foreach \j in {1,...,12} {
				\pgfmathsetmacro\z{\j-6}
				\draw[black!50] (inputcenter\i) ++(\z*\dx, 0.5*\h + \z*\dy) -- ++(0, -\h);
			}
			\draw[black] (inputcenter\i)
			++(\dx*6, 0.5*\h + \dy*6) -- ++(0, -\h) -- ++(-12*\dx, -12*\dy) -- ++(0, \h) --cycle;
			
			\path (m-1-13.base) +(0, 2+\y) coordinate (predictioncenter\i);
			\draw[black!50] (predictioncenter\i) ++(0,0.5*\h) -- ++(0,-\h);
			\draw[black] (predictioncenter\i)
			++(\dx, 0.5*\h + \dy) -- ++(0,-\h) -- ++(-2*\dx,-2*\dy) -- ++(0,\h) -- cycle;
			
			\pgfmathsetmacro\z{0.5*(\n-\h)*\y}
			\draw[<-, blue!60, shorten >=0.0cm, shorten <=0.4cm] (sketchcenter0)
			++(0,\z) -- (inputcenter\i);
			\draw[->, blue!60, shorten <=0.1cm, shorten >=0.3cm] (sketchcenter5)
			++(0,\z) -- (predictioncenter\i);
		}
		
		\path (m-1-1.base) +(-2,0) node[anchor=base] {\small Features};
		
		\end{tikzpicture}
	\end{center}
	\caption{Layers of a general neural network for semi-supervised learning.}
	\label{fig:nn_ssl}
\end{figure}
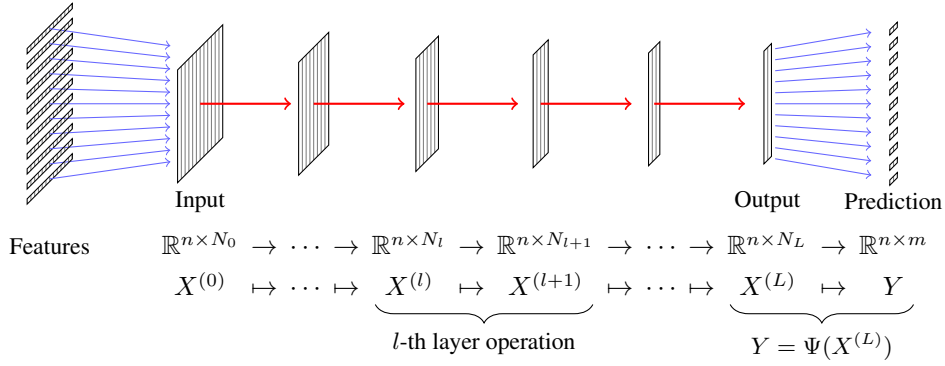

\subsection{Filter functions and their low-rank equivalents}
\label{sec:app:gcn:filters}

This section gives visualizations of the filter functions used in our experiments. First off, let again the eigenvalues of $\mathcal{L}$ be denoted by
\[
0 = \lambda_1 < \lambda_2 \leq \ldots \leq \lambda_n < 2.
\]
Then, the three main filter functions and their corresponding kernel matrices are given by
\begin{align}
&\text{Linear:} & 
\varphi(\lambda) &= 1 - \frac{\lambda}{\lambda_n}, & 
\mathcal{K} &= I - \frac{1}{\lambda_n} \mathcal{L}, 
\label{eq:app:gcn:filters:linear} \\
&\text{Quadratic:} & 
\varphi(\lambda) &= \left(1 - \frac{\lambda}{\lambda_n}\right)^2, 
& \mathcal{K} &= I - \frac{2}{\lambda_n} \mathcal{L} + \frac{1}{\lambda_n^2} \mathcal{L}^2,
\label{eq:app:gcn:filters:quadratic} \\
&\text{Pseudoinverse:} & 
\varphi(\lambda) &= \begin{cases} \frac{\lambda_2}{\lambda} & \text{if $\lambda > 0$,} \\ 0 & \text{else,} \end{cases} & 
\mathcal{K} &= \lambda_2 \mathcal{L}^\dagger.
\label{eq:app:gcn:filters:pseudoinverse}
\end{align}
Note that this definition of the pseudoinverse filter differs from the one introduced in our paper by a factor of $\lambda_2$.
Scaling a filter function does not change the spanned function space and in practice it will be compensated for by reciprocally scaled weight matrices $\Theta$. For comparability, we would like to scale all filter functions in such a way that the maximum entry in $\varphi(\Lambda)$ is 1.

For a small target rank $r$, the low-rank approximations to these filters can be set up by identifying the ``dominant'' $r$ eigenvalues, i.e., those where $|\varphi(\lambda)|$ is largest. For the linear and quadratic function, $\lambda_1$ through $\lambda_r$ are dominant. For the pseudoinverse function, however, $\varphi(\lambda_1)$ is zero, so $\lambda_2$ through $\lambda_{r+1}$ are dominant. The resulting low-rank filters are given by
\begin{align}
&\text{Linear:} & 
\varphi(\lambda) &= \begin{cases} 1 - \frac{\lambda}{\lambda_n} & \text{if $\lambda \in [0,\lambda_r]$,} \\ \makebox[5.1em][l]{0} & \text{else,} \end{cases}
\label{eq:app:gcn:filters:lowranklinear} \\
&\text{Quadratic:} & 
\varphi(\lambda) &= \begin{cases} \left(1 - \frac{\lambda}{\lambda_n}\right)^2 & \text{if $\lambda \in [0,\lambda_r]$,} \\ \makebox[5.1em][l]{0} & \text{else,} \end{cases}
\label{eq:app:gcn:filters:lowrankquadratic} \\
&\text{Pseudoinverse:} & 
\varphi(\lambda) &= \begin{cases} \frac{\lambda_2}{\lambda} & \text{if $\lambda \in [\lambda_2,\lambda_{r+1}]$,} \\ \makebox[5.1em][l]{0} & \text{else,} \end{cases} & 
\label{eq:app:gcn:filters:lowrankpseudoinverse}
\end{align}
Note that because all filters are only evaluated in eigenvalues, the case differentiations
\[
\lambda \in \{\lambda_2,\ldots,\lambda_{r+1}\}, \qquad \lambda \in [\lambda_2, \lambda_{r+1}], \qquad \lambda \in (0, \lambda_{r+1}]
\]
are all equivalent.

Figure~\ref{fig:app:gcn:filters} shows the filter functions from \eqref{eq:app:gcn:filters:linear} through \eqref{eq:app:gcn:filters:lowrankpseudoinverse} evaluated on an exemplary eigenvalue set depicted by the gray vertical lines.

\begin{figure}
	\centering
	\subfloat[Full-rank filters]{
		\begin{tikzpicture}
		\def\lambdatwo{0.3}
		\def\lambdar{0.7}
		\def\lambdan{1.5}
		\begin{axis}[
		width=0.55\textwidth,
		thick,
		xmin=0, xmax=2,
		ymin=0, ymax=1,
		xtick={0,\lambdatwo,\lambdar,1,\lambdan,2},
		xticklabels={$0$, $\lambda_2$, $\lambda_r$, $1$, $\lambda_n$, 2},
		ytick={0,1},
		tick align=outside,
		legend cell align=left,
		legend style = {font=\small},
		major tick style = {thick, black},
		]
		
		\foreach \x in {0.3, 0.4, 0.5, 0.6, 0.65, 0.7, 0.75, 0.8, 0.85, 0.87, 0.9, 0.92, 0.94, 0.97, 1, 1.03, 1.06, 1.1, 1.15, 1.2, 1.3, 1.4, 1.5}{
			\addplot[black!30,thin, forget plot] coordinates {(\x,0) (\x, 1)};
		}
		
		\addplot[black!70, thin, domain=0:2, samples=101, forget plot] {1-x/\lambdan};
		\addplot[black!70, thin, domain=0:2, samples=101, forget plot] {(1-x/\lambdan)^2};
		\addplot[black!70, thin, domain=\lambdatwo:2, samples=101, forget plot] {\lambdatwo/x};
		
\addplot[blue, mark=+, only marks] table[x=x, y=yLinear] {filtervalues.csv};
		\addlegendentry{Linear \eqref{eq:app:gcn:filters:linear}};
		\addplot[red, mark=x, only marks] table[x=x, y=yQuadratic] {filtervalues.csv};
		\addlegendentry{Quadratic \eqref{eq:app:gcn:filters:quadratic}};
		\addplot[amber, mark=o, only marks] table[x=x, y=yPinv] {filtervalues.csv};
		\addlegendentry{Pseudoinverse \eqref{eq:app:gcn:filters:pseudoinverse}};
		
		\end{axis}
		\end{tikzpicture}
	}
	\hspace{-0.8em}
	\subfloat[Low-rank filters] {
		\begin{tikzpicture}
		\def\lambdatwo{0.3}
		\def\lambdar{0.7}
		\def\lambdarone{0.75}
		\def\lambdan{1.5}
		\pgfmathsetmacro\pinvvalue{\lambdatwo/\lambdarone}
		\pgfmathsetmacro\linearvalue{1-\lambdar/\lambdan}
		\begin{axis}[
		width=0.55\textwidth,
		thick,
		xmin=0, xmax=2,
		ymin=0, ymax=1,
		xtick={0,\lambdatwo,\lambdar,1,\lambdan,2},
		xticklabels={$0$, $\lambda_2$, $\lambda_r$, $1$, $\lambda_n$, 2},
		ytick={0,1},
		tick align=outside,
		legend cell align=left,
		legend style = {font=\small},
		major tick style = {thick, black},
		]
		
		\foreach \x in {0.3, 0.4, 0.5, 0.6, 0.65, 0.7, 0.75, 0.8, 0.85, 0.87, 0.9, 0.92, 0.94, 0.97, 1, 1.03, 1.06, 1.1, 1.15, 1.2, 1.3, 1.4, 1.5}{
			\addplot[black!30,thin, forget plot] coordinates {(\x,0) (\x, 1)};
		}
		
		\addplot[black!70, thin, domain=0:2, samples=101, forget plot] {1-x/\lambdan};
		\addplot[black!70, thin, domain=0:2, samples=101, forget plot] {(1-x/\lambdan)^2};
		\addplot[black!70, thin, domain=\lambdatwo:2, samples=101, forget plot] {\lambdatwo/x};
		
		\addplot[blue, mark=+, only marks] table[x=x, y=yLinearLowrank] {filtervalues.csv};
		\addlegendentry{Linear \eqref{eq:app:gcn:filters:lowranklinear}};
		\addplot[red, mark=x, only marks] table[x=x, y=yQuadraticLowrank] {filtervalues.csv};
		\addlegendentry{Quadratic \eqref{eq:app:gcn:filters:lowrankquadratic}};
		\addplot[amber, mark=o, only marks] table[x=x, y=yPinvLowrank] {filtervalues.csv};
		\addlegendentry{Pseudoinverse \eqref{eq:app:gcn:filters:lowrankpseudoinverse}};
		
		\end{axis}
		\end{tikzpicture}
	}
	\caption{Filter functions and their low-rank equivalents}
	\label{fig:app:gcn:filters}
\end{figure}
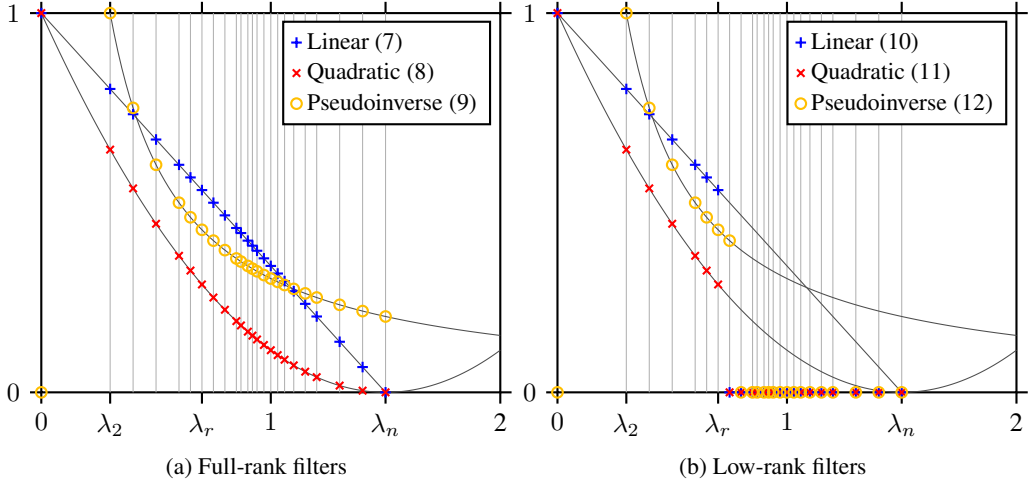

\subsection{Proof of best-approximation property}
\label{sec:app:gcn:bestapproximation}

In this section, we would like to formalize the claims of Section~2.3.4 of our paper in a lemma.

\begin{lemma}
	Let $\Lambda = \diag(\lambda_1, \ldots, \lambda_n) \in \mathbb{R}^{n\times n}$ be a diagonal matrix and $U \in \mathbb{R}^{n \times n}$ be an orthogonal matrix, i.e., $U^T U = I$. For a given arbitrary function $\varphi : \mathbb{R} \to \mathbb{R}$ and a given target rank $1 < r < n$, define a second function via
	\[
	\tilde{\varphi}(\lambda) = \begin{cases} \varphi(\lambda) & \text{if $\lambda$ is one of the $r$ eigenvalues with largest $|\varphi(\lambda)|$ value,} \\ 0 & \text{else.} \end{cases}
	\]
	Then $\tilde{\mathcal{K}} = U \tilde{\varphi}(\Lambda) U^T$ is a best rank-$r$ approximation to $\mathcal{K} = U \varphi(\Lambda) U^T$ in the Frobenius norm, i.e., $\rank(\tilde{\mathcal{K}}) = r$ and 
	\[
	\|\mathcal{K} - \tilde{\mathcal{K}}\|_{\mathrm{F}} = \min \{ \|\mathcal{K} - A\|_{\mathrm{F}} : A \in \mathbb{R}^{n \times n}, \ \rank(A) = r \},
	\]
	where $\|\cdot\|_{\mathrm{F}}$ is the square root of the sum of all squared entries of a matrix.
\end{lemma}

\begin{proof}
	The lemma is proven by stating the singular value decompositions of $\mathcal{K}$ and $\tilde{\mathcal{K}}$ and then following, e.g., \cite[Section 7.4.2]{hornjohnson85}.
	Let $u_i \in \mathbb{R}^n$ denote the $i$-th column of $U$. Also let $i_1,\ldots,i_n$ be a permutation of $1,\ldots,n$ such that
	\[
	|\varphi(\lambda_{i_1})| \geq \ldots \geq |\varphi(\lambda_{i_n})|.
	\]
	The SVD of $\mathcal{K}$ is given by
	\begin{align*}
	\mathcal{K} = U \varphi(\Lambda) U^T &= \begin{bmatrix} u_{i_1} & \cdots & u_{i_n} \end{bmatrix} \begin{bmatrix} \varphi(\lambda_{i_1}) && \\ &\ddots& \\ && \varphi(\lambda_{i_n}) \end{bmatrix} \begin{bmatrix} u_{i_1}^T \\ \vdots \\ u_{i_n}^T \end{bmatrix} \\
	&= \begin{bmatrix} u_{i_1} & \cdots & u_{i_n} \end{bmatrix} \begin{bmatrix} |\varphi(\lambda_{i_1})| && \\ &\ddots& \\ && |\varphi(\lambda_{i_n})| \end{bmatrix} \begin{bmatrix} \pm u_{i_1}^T \\ \vdots \\ \pm u_{i_n}^T \end{bmatrix} =: \hat{U} \Sigma \hat{V}^T,
	\end{align*}
	where the signs of each $\pm u_i$ in $\hat{V}$ are chosen to be equal to the sign of the corresponding $\varphi(\lambda_i)$. This does indeed constitute an SVD, as $\hat{U}$ and $\hat{V}$ are still orthogonal and the diagonal elements of $\Sigma$ are non-negative and sorted non-increasingly.
	
	For the low-rank approximation, we have $\tilde{\varphi}(\lambda_{i_j}) = \varphi(\lambda_{i_j})$ for $j \leq r$ and $\tilde{\varphi}(\lambda_{i_j}) = 0$ for $j > r$. Hence the analogous SVD of $\tilde{\mathcal{K}}$ is
	\[
	\tilde{\mathcal{K}} = \begin{bmatrix} u_{i_1} & \cdots & u_{i_n} \end{bmatrix} \begin{bmatrix} |\varphi(\lambda_1)| &&&& \\ &\ddots&&& \\ && |\varphi(\lambda_r)| && \\ &&& 0 & \\ &&&& \ddots \end{bmatrix} \begin{bmatrix} \pm u_{i_1}^T \\ \vdots \\ \pm u_{i_n}^T \end{bmatrix} =: \hat{U} \tilde{\Sigma} \hat{V}^T.
	\]
	This formally shows the rank of $\tilde{\mathcal{K}}$ to be $r$.
	
	Now, it is stated in \cite[Corollary 7.4.1.5a]{hornjohnson85} that for any two matrices $A,B \in \mathbb{R}^{n\times n}$, it holds
	\begin{equation}
	\label{eq:app:gcn:frobeniusinequality}
	\|B-A\|_{\mathrm{F}}^2 \geq \sum_{j=1}^n (\sigma_j(B) - \sigma_j(A))^2,
	\end{equation}
	with equality if and only if $\trace(BA^T) = \sum_{j=1}^n \sigma_j(B) \sigma_j(A)$.
	Here $\sigma_j(\cdot)$ denote the non-increasingly sorted singular values of a matrix, $\sigma_1(\cdot) \geq \ldots \geq \sigma_n(\cdot)$.
	
	If $A \in \mathbb{R}^{n \times n}$ is an arbitrary matrix with rank $r$, we have $\sigma_j(A) = 0$ for $j>r$ and it follows that
	\[
	\|\mathcal{K}-A\|_{\mathrm{F}}^2 \geq \sum_{j=1}^r \underbrace{(\sigma_j(\mathcal{K}) - \sigma_j(A))^2}_{\geq 0} + \sum_{j=r+1}^n \sigma_j(\mathcal{K})^2 \geq \sum_{j=r+1}^n |\varphi(\lambda_{i_j})|^2.
	\]
	For $A = \tilde{\mathcal{K}}$, we see
	\[
	\trace(\mathcal{K}\tilde{\mathcal{K}}^T) = \trace(\hat{U} \Sigma \underbrace{\hat{V}^T \hat{V}}_{=I} \tilde{\Sigma} \hat{U}^T) = \trace(\Sigma \tilde{\Sigma}) = \sum_{j=1}^n \sigma_j(\mathcal{K}) \sigma_j(\tilde{\mathcal{K}})
	\]
	since the trace is invariant under orthonormal similarity transforms. For this reason, we have equality in \eqref{eq:app:gcn:frobeniusinequality} and thus
	\[
	\|\mathcal{K} - \tilde{\mathcal{K}}\|_{\mathrm{F}}^2 = \sum_{j=1}^r \underbrace{(\sigma_j(\mathcal{K}) - \sigma_j(\tilde{\mathcal{K}}))^2}_{= 0}\ + \sum_{j=r+1}^n \sigma_j(\mathcal{K})^2 = \sum_{j=r+1}^n |\varphi(\lambda_{i_j})|^2.
	\]
	This shows that $\|\mathcal{K} - \tilde{\mathcal{K}}\|_{\mathrm{F}} \leq \|\mathcal{K} - A\|_{\mathrm{F}}$ for any rank-$r$ matrix $A$.
\end{proof}

\subsection{Smoothed Laplacians, graphs with loops, and the \textit{re-normalization trick}}
\label{sec:app:gcn:smooth}

In \cite{kipf17}, a GCN with kernel matrix 
\begin{equation}
\label{eq:app:kipfkernel}
\mathcal{K} = (D + I)^{-1/2} (W + I) (D + I)^{-1/2}
\end{equation}
was used, which was motivated as a \emph{re-normalization} of the kernel matrix 
\[
2I - \mathcal{L} = I + D^{-1/2} W D^{-1/2}
\]
originating from the filter function $\varphi(\lambda) = 2 - \lambda$.

Later, it was observed that the same matrix can be obtained from the filter function $\varphi(\lambda) = 1 - \lambda$ if we replace $\mathcal{L}$ by an in some way \emph{smoothed} graph Laplacian matrix \cite{li18}. Here we will introduce smoothing in a very general way. Given a diagonal matrix $S = \diag(s_1,\ldots,s_n)$ with positive entries, we define the smoothed graph Laplacian as
\begin{equation} \label{eq:glo:smooth}
\begin{aligned}
\mathcal{L}_S
&= (D + S)^{-1/2} (D - W) (D + S)^{-1/2} \\
&= I - (D+S)^{-1/2} (W+S) (D+S)^{-1/2} 
\end{aligned}
\end{equation}
We refer to $S$ as the \emph{smoother} matrix. By choosing the special identity smoother $S=I$, we see that \eqref{eq:app:kipfkernel} indeed satisfies $\mathcal{K} = I - \mathcal{L}_I$.

The definition \eqref{eq:glo:smooth} can also be understood as the traditional graph Laplacian of a graph with adjacency matrix $W+S$. This differs from the original setting because usually the Laplacian matrix is only defined for graphs without loops, where the diagonals of $W$ and $\mathcal{L}$ are $0$ and $1$, respectively. Loops are edges with the same start and end node, leading to different diagonal entries in $W$ and thus $\mathcal{L}$. Using a smoothed Laplacian can now be understood as augmenting the edge set by loops with weight $s_i$ around each node $i$.

The name \emph{smoothed graph Laplacian} was also introduced in a different context in \cite{langone16}. In our setting, we understand \eqref{eq:glo:smooth} as a type of smoothing because, e.g., $S=I$ leads to a smoother distribution in the entries of the normalizer $(D+S)^{-1/2}$, i.e., its entries are relatively closer to each other than in $D^{-1/2}$. In the graph context, introducing loops with the same weight to each node leads to a smoother distribution of node degrees. We will stick with this name even though we acknowledge that specific choices of $S \neq I$ may even lead to less smooth behaviour.

\section{Hypergraphs}
\label{sec:app:hypergraph}

\subsection{Efficient evaluations of convolutional layer operations with the hypergraph Laplacian for few hyperedges}
\label{sec:app:hypergraph:efficient}

Consider a hypergraph with a small number of hyperedges, i.e., $|E| \ll n$.
In this section, we will describe how to exploit the special structure of the hypergraph Laplacian (i.e. the graph Laplacian with hypergraph smoothing). Our goal is to decrease the evaluation cost of each layer operation, possibly by avoiding the setup of the full Laplacian and kernel matrix altogether.
Recall that the hypergraph Laplacian can be written as
\begin{equation}
\label{eq:app:hypergraph:lowrankstructure}
\mathcal{L} = I - \tilde{H} \tilde{H}^T \quad \text{with} \quad \tilde{H} = D_V^{-1/2} H W_E^{1/2} D_E^{-1/2} \in \mathbb{R}^{n \times |E|}.
\end{equation}
Consider a full singular value decomposition (SVD) of $\tilde{H}$, i.e.
\begin{equation*}
\tilde{H} = U \Sigma V^T
\end{equation*}
where $U \in \mathbb{R}^{n \times n}$ and $V \in \mathbb{R}^{|E|\times |E|}$ are orthogonal matrices and $\Sigma \in \mathbb{R}^{n \times |E|}$ holds the singular values of $\tilde{H}$ in descending order on its main diagonal. A \textit{thin} SVD can be obtained by splitting the SVD into
\begin{equation*}
\tilde{H} = U \Sigma V^T = \begin{bmatrix} U_R & \tilde{U} \end{bmatrix} \begin{bmatrix} \Sigma_R & 0 \\ 0 & 0 \end{bmatrix} \begin{bmatrix} V_R^T \\ \tilde{V}^T \end{bmatrix} = U_R \Sigma_R V_R^T,
\end{equation*}
where $R$ is the rank of $\tilde{H}$, $U_R \in \mathbb{R}^{n \times R}$ and $V_R \in \mathbb{R}^{|E|\times R}$ hold the first $R$ columns of $U$ and $V$, and $\Sigma_R \in \mathbb{R}^{R\times R}$ is the diagonal matrix of the non-zero singular values. Note that in most applications we can assume $R = |E|$, leading to $V=V_R$ and $\Sigma = \begin{bmatrix} \Sigma_R^T & 0 \end{bmatrix}^T$. 

The SVD immediately gives us the eigenvalues of $\mathcal{L}$ due to
\begin{equation}
\label{eq:app:hypergraph:fulleigenvalues}
\begin{aligned}
\mathcal{L} &= I - \tilde{H} \tilde{H}^T = I - U \Sigma V^T V \Sigma^T U^T = U U^T - U \Sigma \Sigma^T U^T = U \left(I - \Sigma \Sigma^T \right) U^T \\
&= U \left(I - \begin{bmatrix} \Sigma_R^2 & 0 \\ 0 & 0 \end{bmatrix} \right) U^T = U \begin{bmatrix} I-\Sigma_R^2 & 0 \\ 0 & I \end{bmatrix} U^T.\end{aligned}
\end{equation}
In other words, each of the $R$ nonzero singular values $\sigma_i$ of $\tilde{H}$ produces an eigenvalue $\lambda_i = 1 - \sigma_i^2$. The corresponding eigenvector is the $i$-th column of $U_R$. The remaining $n-R$ eigenvalues are all $1$ and their associated eigenvectors are the columns of $\tilde{U}$.

Even though we know the full eigenvalue decomposition \eqref{eq:app:hypergraph:fulleigenvalues} of $\mathcal{L}$, it is beneficial to work only with \begin{equation*}
\mathcal{L} = I - U \Sigma \Sigma^T U^T = I - U \begin{bmatrix}  \Sigma_R^2 & 0 \\ 0 & 0 \end{bmatrix} U^T = I - U_R \Sigma_R^2 U_R^T.
\end{equation*}
This conveniently exploits the structure and it is cheap to compute, as the thin SVD of $\tilde{H}$ can be computed via the eigenvalue decomposition of the smaller matrix $\tilde{H}^T \tilde{H}$. Let $\mu_i$ be an eigenvalue of that matrix with eigenvector $v_i \in \mathbb{R}^{|E|}$. Then $\sigma_i = \sqrt{\mu_i}$ is a singular value of $\tilde{H}$ with right singular value $v_i$ and left singular value $u_i = \frac{1}{\sigma_i} \tilde{H} v_i$. We can skip this step and directly compute the eigenpair of $\mathcal{L}$ through $\lambda_i = 1-\mu_i$ and $u_i = \frac{1}{\sqrt{\mu_i}} \tilde{H} v_i$.

\subsubsection{Affine linear filters}
\label{sec:app:hypergraph:efficient:linear}

In the case of $\varphi(\lambda) = a_0 + a_1 \lambda$, it is clear that the corresponding kernel matrix is
\begin{equation}
\label{eq:app:hypergraph:efficient:linear}
\mathcal{K} = a_0 I + a_1 \mathcal{L} = (a_0 + a_1) I - a_1 \tilde{H} \tilde{H}^T.
\end{equation}
This structure as a linear combination of the identity and a low-rank matrix is already easy to exploit for cheap matrix products.

\subsubsection{Higher-order polynomial filters}
\label{sec:app:hypergraph:efficient:polynomial}

For higher-order polynomials $\varphi(\lambda) = \sum_{j=0}^{p} a_j \lambda^j$, the evaluation costs increases slightly. First, we can rewrite the polynomial using a Taylor expansion around $\lambda=1$,
\begin{equation}
\label{eq:app:hypergraph:efficient:polynomial}
\varphi(\lambda) = \sum_{j=0}^{p} b_j (\lambda - 1)^j \quad \text{with} \quad b_j = \frac{1}{j!} \frac{\mathrm{d}^j \varphi}{\mathrm{d} \lambda^j}(1) = \frac{1}{j!} \sum_{i=j}^{p} \frac{i!}{(i-j)!} a_i = \sum_{i=j}^{p} \binom{i}{j} a_i,
\end{equation}
where $\binom{i}{j}$ denotes the binomial coefficient. Hence the kernel matrix can be written as
\begin{equation}
\mathcal{K} = \sum_{j=0}^{p} a_j \mathcal{L}^j = \sum_{j=0}^{p} b_j (-\tilde{H} \tilde{H}^T)^j = b_0 I + \tilde{H} \underbrace{\left( \sum_{j=1}^p (-1)^j b_j (\tilde{H}^T \tilde{H})^{j-1} \right)}_{=: M} \tilde{H}^T.
\end{equation}
An efficient way to evaluate matrix products with this structure is to have a precomputation step that produces $b_0$, $\tilde{H}$, and the ``small'' matrix $M \in \mathbb{R}^{|E|\times |E|}$. This way, independently of the degree $K$, the training cost is  only slightly higher than with linear functions.

\subsubsection{Arbitrary full-rank filters}
\label{sec:app:hypergraph:efficient:fullrank}

For non-polynomial filter functions $\varphi$, we can use the SVD-based eigenvalue decomposition and the fact that most eigenvalues are $1$ to obtain
\begin{equation}
\begin{aligned}
\mathcal{K} &= U \varphi(\Lambda)U^T = U \begin{bmatrix} \varphi(\Lambda_R) & \\ & \varphi(I) \end{bmatrix} U^T = U \varphi(I) U^T + U \begin{bmatrix} \varphi(\Lambda_R)-\varphi(I) & \\ & 0 \end{bmatrix} U^T \\
&= \varphi(1) I + U_R \underbrace{\big(\varphi(\Lambda_R) - \varphi(1) I \big)}_{=: M} U_R^T,
\end{aligned}
\end{equation}
where $\Lambda_R$ again denotes the diagonal matrix holding the $R$ eigenvalues smaller than one.
Similar to polynomials, the kernel matrix is a linear combination of the identity and a low-rank matrix. The inner part $M$ even is diagonal, holding the values of $\varphi(\lambda_i) - \varphi(1)$ for all eigenvalues smaller than one. Moreover, we see that if $\varphi(1) = 0$ (which is desirable due to the third design principle), $\mathcal{K} = U_R \varphi(\Lambda_R) U_R^T$ itself only has rank $R$ and each such $\varphi$ may be considered a low-rank filter.

This way, we only need the SVD of $\tilde{H}$ to employ arbitrary full-rank filter functions. One example is the scaled pseudoinverse kernel from \eqref{eq:app:gcn:filters:pseudoinverse}, which leads to the kernel matrix
\begin{equation}
\mathcal{K} = \lambda_2 I + U_R M U_R^T \qquad \text{with} \qquad M = \lambda_2 \begin{bmatrix} -1 &&& \\ & \frac{1}{\lambda_2}-1 && \\ && \ddots & \\ &&& \frac{1}{\lambda_R} - 1 \end{bmatrix}.
\end{equation}

\subsubsection{Low-rank filters}
\label{sec:app:hypergraph:efficient:lowrank}

If $\varphi(\lambda)$ is non-zero only for the smallest $r$ eigenvalues of $\mathcal{L}$, and $r < R$, then we only need to compute the largest $r$ singular values of $\tilde{H}$. This can further speed up computations if the target rank is significantly smaller than $|E|$. Otherwise it is a valid approach to simply compute all $R$ singular values, of which only the largest $r$ are used.

\subsection{Interpretation as a smoothed graph Laplacian}
\label{sec:app:hypergraph:interpretation}

Following \cite{zhou06}, we have introduced the hypergraph Laplacian as
\begin{equation}
\label{eq:app:hypergraph:zhou}
\mathcal{L}_{\text{Hypergraph}} = I - D_V^{-1/2} H W_E D_E^{-1} H^T D_V^{-1/2}.
\end{equation}
This structure appears similar to the traditional graph Laplacian with ``adjacency matrix'' $H W_E D_E^{-1} H^T$ and degree matrix $D_V$. Indeed, we see via 
\[
\sum_{j=1}^n (H W_E D_E^{-1} H^T)_{ij} = \sum_{j=1}^n \sum_{e \in E} h_{ie} \frac{w_e}{|e|} h_{je} = \sum_{e \in E} h_{ie} \frac{w_e}{|e|} \underbrace{\sum_{j=1}^n h_{je}}_{|e|} = \sum_{e \in E} h_{ie} w_e = (D_V)_{ii}
\]
that $D_V$ holds the row sums of $H W_E D_E^{-1} H^T$.
The only aspect that differs from the graph setting is the fact that the diagonal of $H W_E D_E^{-1} H^T$ is nonzero. This means that in terms of the Laplacian, a hypergraph is identical to a traditional graph with a specific set of self-loops.
This is sometimes referred to as a \emph{clique expansion} of the hypergraph \cite{yadati19}.
We will emphasize this connection by introducing slightly different notation for the graph adjacency matrix and degree matrix,
\begin{align}
W_H &= H W_E D_E^{-1} H^T, & (W_H)_{ij} &= \sum_{e \in E} h_{ie} h_{je} \frac{w_e}{|e|}, \nonumber \\
D_H &= \diag(d_{H,1},\ \ldots,\ d_{H,n}), & d_{H,i} &= \sum_{j = 1}^n (W_H)_{ij} = \sum_{e \in E} h_{ie} w_e. \nonumber \\
\intertext{
	Note that $D_H$ and $d_{H,i}$ were previously called $D_V$ and $d_i$.
	Following Section~\ref{sec:app:gcn:smooth}, we will think of this graph with loops as a specially smoothed version of a graph without loops, i.e. $W_H = W + S_H$ with a diagonal-free adjacency matrix $W$ and a diagonal smoother $S_H$. These matrices and their corresponding graph degree matrix are given by
}
\label{eq:app:hypergraph:loopweights}
S_H &= \diag(s_{H,1},\ \ldots, s_{H,n}), & s_{H,i} &= (W_H)_{ii} = \sum_{e \in E} h_{ie} \frac{w_e}{|e|},\\
W &= W_H - S_H, & W_{ij} &= \begin{cases} \sum_{e \in E} h_{ie} h_{je} \frac{w_e}{|e|} & \text{if $i\neq j$,} \\ 0 & \text{if $i=j$,} \end{cases} \\
D &= D_H - S_H = \diag(d_1,\ \ldots,\ d_n), & d_i &= d_{H,i} - s_{H,i} = \sum_{e \in E} h_{ie} w_e \left(1 - \frac{1}{|e|}\right).
\end{align}
Inserting these definitions in \eqref{eq:app:hypergraph:zhou} indeed gives us the formula 
\[
\mathcal{L}_{\text{Hypergraph}} = I - (D+S_H)^{-1/2} (W+S_H) (D+S_H)^{-1/2} = \mathcal{L}_{S_H},
\]
which is equal to the smoothed graph Laplacian $\mathcal{L}_{S_H}$ \eqref{eq:glo:smooth}.

\subsection{Arbitrarily smoothed hypergraph Laplacians and the \emph{re-normalization trick}}
\label{sec:app:hypergraph:arbitrarysmoothers}

Note that we do not have to use this special hypergraph smoother matrix, but can instead transform the hypergraph into a graph, remove its loops to obtain the adjacency matrix $W$, and optionally apply any arbitrary smoother matrix $S$ afterwards. This might be $S=I$ as in Section~\ref{sec:app:gcn:smooth} or $S=S_H$ as above, or even a linear combination of the two. That includes the application of the \emph{re-normalization trick} from \cite{kipf17} (cf. Section~\ref{sec:app:gcn:smooth}) in the hypergraph context, as used in \cite{feng19}, resulting in the Laplacian
\begin{equation}
\label{eq:app:hypergraph:combinedsmoothing}
\begin{aligned}
\tilde{\mathcal{L}} &= I - (D_H+I)^{-1/2} (H W_E D_E^{-1} H^T + I) (D_H+I)^{-1/2} &&\\
&= I - (D+S)^{-1/2} (W+S) (D+S)^{-1/2} = \mathcal{L}_S && \text{with } S = S_H + I.
\end{aligned}
\end{equation}

Choosing an arbitrary smoother $S \neq S_H$ destroys the structure of the hypergraph Laplacian as an identity matrix minus a symmetric positive definite, possibly low-rank matrix. However, a similar structure can be exploited in the form of
\begin{equation}
\label{eq:app:hypergraph:laplaciansplitting}
\begin{aligned}
\mathcal{L}_S  &= I - (D + S)^{-1/2} (W+S)(D+S)^{-1/2} \\
&= I - (D_V + S - S_H)^{-1/2} (H W_E D_E^{-1} H^T + S - S_H) (D_V + S - S_H)^{-1/2} \\
&= \underbrace{I - (D_V + S - S_H)^{-1} (S - S_H)}_{\text{diagonal matrix}} \\
&\qquad\qquad - \underbrace{(D_V+S-S_H)^{-1/2} H W_E D_E^{-1} H^T (D_V+S-S_H)^{-1/2}}_{\text{positive semi-definite matrix with rank equal to $H$}}.
\end{aligned}
\end{equation}
If we simply set up the full matrix, the cost for a naive evaluation of matrix-vector-products with $\mathcal{L}_S$ is asymptotically $\mathcal{O}(n^2)$. By exploiting the above structure, that cost can be reduced to $\mathcal{O}(\nnz(H))$, where $\nnz$ denotes the number of non-zero entries in a matrix. This is important especially if there are distinctly fewer hyperedges than nodes, i.e. $|E| \ll n$, which is the case in our experimental settings. Then the cost can be estimated as $\mathcal{O}(n\cdot |E|)$.

In the case of $|E| \ll n$, this structure can be exploited especially in the case of affine linear filter functions. Similar to \eqref{eq:app:hypergraph:lowrankstructure}, let
\[
\mathcal{L}_S = \mathcal{D} - \mathcal{Z} \mathcal{Z}^T
\]
denote the structure of \eqref{eq:app:hypergraph:laplaciansplitting}, where the diagonal matrix $\mathcal{D} \in \mathbb{R}^{n \times n}$ and the \emph{tall-and-skinny} matrix $\mathcal{Z} \in \mathbb{R}^{n \times |E|}$ depend on the smoother $S$.
Then the kernel matrix $\mathcal{K}$ of a linear filter $\varphi(\lambda) = a_0 + a_1 \lambda$ can equally be written as a diagonal matrix minus a low-rank one via
\[
\mathcal{K} = a_0 I - a_1 \mathcal{L} = (a_0 I + a_1 \mathcal{D}) - a_1 \mathcal{Z} \mathcal{Z}^T,
\]
yielding the same computational benefits as in \eqref{eq:app:hypergraph:efficient:linear} for the hypergraph Laplacian.

When using higher-order polynomials or low-rank filter functions, however, arbitrary smoothers have a slight disadvantage over hypergraph smoothing. For polynomials, the structure of Section~\ref{sec:app:hypergraph:efficient:polynomial} cannot be retained. For example, with a quadratic filter $\varphi(\lambda) = a_0 + a_1 \lambda + a_2 \lambda^2$, we gain the structure
\[
\begin{aligned}
\mathcal{K} &= a_0 I + a_1 (\mathcal{D} - \mathcal{Z} \mathcal{Z}^T) + a_2 (\mathcal{D}^2 - \mathcal{D} \mathcal{Z}\mathcal{Z}^T - \mathcal{Z} \mathcal{Z}^T \mathcal{D} + \mathcal{Z} \mathcal{Z}^T \mathcal{Z} \mathcal{Z}^T) \\
&= \underbrace{(a_0 I + a_1 \mathcal{D} + a_2 \mathcal{D}^2)}_{\text{diagonal matrix}} + \underbrace{\begin{bmatrix} \mathcal{Z} & \mathcal{D} \mathcal{Z} \end{bmatrix} \begin{bmatrix} (a_2 \mathcal{Z}^T \mathcal{Z} - a_1 I) & -a_2 I \\ -a_2 I & 0 \end{bmatrix} \begin{bmatrix} \mathcal{Z} & \mathcal{D} \mathcal{Z} \end{bmatrix}^T,}_{\text{matrix with rank $2|E|$}}
\end{aligned}
\]
which is similar to \eqref{eq:app:hypergraph:efficient:polynomial} but with twice the rank. Analogously, the rank appearing in the expression for a polynomial of degree $p$ will be $p|E|$. There is no point in using this structure instead of simple $p$-fold multiplication with $\mathcal{L}$  in each evaluation of the layer operation. Consequently, the cost of matrix products with $\mathcal{K}$ scales with the polynomial degree.

For non-polynomial filters, recall that the eigenvalues of the hypergraph Laplacian can be computed via the singular values of a normalized incidence matrix, which is much cheaper if $|E| \ll n$. The same cannot be achieved for other smoothers $S \neq S_H$, since $\mathcal{D}$ will generally be a non-identity diagonal matrix and thus the eigenvalues of $\mathcal{L}$ and singular values of $\mathcal{Z}$ will not be connected. Most importantly, unlike Section~\ref{sec:app:hypergraph:efficient:fullrank}, there will generally not be an eigenvalue with multiplicity $n-|E|$, but possibly $n$ different eigenvalues. That being said, we can still exploit the structure of \eqref{eq:app:hypergraph:laplaciansplitting} in an iterative eigenvalue computation scheme such that we again do not have to set up the full matrix.

\section{Experiments}
\label{sec:app:experiments}

\subsection{Architecture}
\label{sec:app:experiments:architecture}

The full architecture of the GCN and L-GCN setups can be written as
\begin{equation*}
\begin{aligned}
&\text{Input:} & X^{(0)} &= X & &\in \mathbb{R}^{n \times N_0} \\
&\text{Hidden layer:} & X^{(1)} &= \sigma\Big( \mathcal{K} X^{(0)} \Theta^{(1)} \Big) &&\in \mathbb{R}^{n \times N_1} \\
&\text{Output:} & X^{(2)} &= \hphantom{\sigma\Big(} \mathcal{K} X^{(1)} \Theta^{(2)} &&\in \mathbb{R}^{n \times N_2} \\
&\text{Prediction:} & Y &= \Psi(X^{(2)}) &&\in \mathbb{R}^{n \times m}
\end{aligned}
\end{equation*}
or in a single step 
\begin{equation}
\label{eq:app:architecture}
Y = \Psi\Big(\mathcal{K} \ \sigma\big(\mathcal{K} \ X \ \Theta^{(1)}\big) \ \Theta^{(2)}\Big),
\end{equation}
where
\begin{itemize}
	\item $n$ is the number of nodes in the dataset.
	\item $N_0,N_1,N_2$ are the numbers of features in the input layer, hidden layer, and output layer, respectively. $N_0$ and $N_2$ are determined by the dataset, while $N_1$ is a free design parameter.
	\item $X$ is the input matrix. For the hypergraph datasets, we will choose $X$ as the hypergraph incidence matrix $H$, i.e. $N_0 = |E|$. The spiral dataset comes with a feature matrix.
	\item $\mathcal{K} = U \varphi(\Lambda) U^T$ is the kernel matrix of the filter function $\varphi$. For low-rank filters, the matrix is never set up directly, but kept in its factorized form $\mathcal{K} = U_r \varphi(\Lambda_r) U_r^T$.
	\item $\Theta^{(1)} \in \mathbb{R}^{N_0 \times N_1}$ and $\Theta^{(2)} \in \mathbb{R}^{N_1 \times N_2}$ are the weight matrices of the matrix. These are the parameters that are determined in the training process, see Section~\ref{sec:app:experiments:training}.
	\item $\sigma$ is the ReLU function $\sigma(x) = \max\{x,0\}$, evaluated element-wise on matrices.
	\item $\Psi$ is the softmax function applied row-wise. 
With this choice, the entries of the output matrix $Y$ contain the computed probabilities for the $i$-th node to belong to the $j$-th class, given by
	\begin{equation}
	\label{eq:app:experiments:softmaxprediction}
	Y_{ij} = \frac{\exp(X^{(2)}_{ij})}{\sum_{k=1}^m \exp(X^{(2)}_{ik})}.
	\end{equation}
	for $i=1,\ldots,n$ and $j=1,\ldots,m$.
	The number of columns in $X^{(2)}$ must be equal to the number of classes, i.e., $N_2 = m$.
\end{itemize}
For the R-GCN setup, the architecture  can similarly be written as
\begin{equation*}
\begin{aligned}
&\text{Input:} & X^{(0)} &= X & &\in \mathbb{R}^{n \times N_0} \\
&\text{Reduction:} & \hat{X}^{(0)} &= U_r^T X^{(0)} &&\in \mathbb{R}^{r \times N_0} \\
&\text{Hidden layer:} & \hat{X}^{(1)} &= \sigma\Big( \varphi(\Lambda_r) \hat{X}^{(0)} \Theta^{(1)} \Big) &&\in \mathbb{R}^{r \times N_1} \\
&\text{Output:} & \hat{X}^{(2)} &= \hphantom{\sigma\Big(} \varphi(\Lambda_r) \hat{X}^{(1)} \Theta^{(2)} &&\in \mathbb{R}^{r \times N_2} \\
&\text{Projection:} & X^{(2)} &= U_r \hat{X}^{(2)} &&\in \mathbb{R}^{n \times N_2} \\
&\text{Prediction:} & Y &= \Psi(X^{(2)}) &&\in \mathbb{R}^{n \times m}
\end{aligned}
\end{equation*}
or in a single step
\begin{equation}
\label{eq:app:reducedarchitecture}
Y = \Psi\Big(U_r \ \varphi(\Lambda_r) \ \sigma\big(\varphi(\Lambda_r) \ U_r^T \ X \ \Theta^{(1)}\big) \ \Theta^{(2)}\Big).
\end{equation}

\subsection{Eigenvalue computation}

At several points in our method, we need to compute eigeninformation of matrices. We will here describe the different computational methods.

\subsubsection{Computing only the smallest eigenvalues}
\label{sec:app:experiments:eigenvalues:which}

For low-rank filters, we always need to compute the dominant $r$ eigenvalues with respect to $\varphi$, i.e., those eigenvalues where the expression $|\varphi(\lambda)|$  is largest. Since all our filters follow the third design principle, these dominant values will always be found among the smallest eigenvalues.

Unfortunately, numerical methods like Krylov subspace iterations are designed to find those eigenvalues with the \emph{largest absolute} value. To compute the smallest ones, these methods are usually implicitly applied to the inverse matrix $\mathcal{L}^{-1}$, solving a system of equations in each iteration. We can avoid this additional complexity by observing that all eigenvalues of $\mathcal{L}$ are contained in the interval $[0,2)$, so our desired values can be obtained more efficiently by computing the largest eigenvalues of $2I - \mathcal{L}$ with a Krylov method. Applying a large enough smoother matrix sometimes decreases the known eigenvalue bound.For example, hypergraph smoothing produces the hypergraph Laplacian with all eigenvalues in $[0,1]$, so we can compute the largest eigenvalues of $I-\mathcal{L}$ instead.

For the low-rank \emph{linear} and \emph{quadratic} filters, we simply need to compute the $r$ smallest eigenvalues of the Laplacian matrix. For the \emph{pseudoinverse} filter, on the other hand, the first eigenvalue $\lambda_1=0$ has a value of $\varphi(\lambda_1) = 0$, so the dominant eigenvalues are $\lambda_2$ through $\lambda_{r+1}$ as described in Section~\ref{sec:app:gcn:filters}. We do this by computing the $r+1$ smallest eigenvalues and then discarding the first one.

\subsubsection{Matrices without exploitable structure}
\label{sec:app:experiments:eigenvalues:nostructure}

In some cases like sparse graphs, hypergraphs with a large number of hyperedges, or smoothed hypergraph Laplacian as in Section~\ref{sec:app:hypergraph:arbitrarysmoothers}, the best approach is to simply set up the full matrix $2I - \mathcal{L}$ and employ a standard method. Arguably the most frequently used algorithm is the implicitly restarted Arnoldi method \cite{lehoucq96}. It computes the largest eigenvalues of a matrix by projection onto a small Krylov subspace. For symmetric matrices, the algorithm reduces to the Lanczos method.

A high performance implementation of the implicitly restarted Arnoldi method is available the \textsc{Arpack}\footnote{\url{https://www.caam.rice.edu/software/ARPACK/}} \cite{arpack98} library, written in \textsc{Fortran}, which is the de-facto standard numerical method for eigenvalue computation. Wrappers exist for various programming languages including \textsc{Matlab} (\texttt{eigs}) and \textsc{Python} (\texttt{scipy.sparse.linalg.eigs}). Our \textsc{Julia} implementation uses the \texttt{eigs} function from the \texttt{Arpack.jl} package.\footnote{\url{https://www.github.com/JuliaLinearAlgebra/Arpack.jl}}

Since Krylov methods only require matrix-vector products with the system matrix, we may provide a way to evaluate these products more efficiently. This is the case for the smoothed hypergraph Laplacians, as described in Section~\ref{sec:app:hypergraph:arbitrarysmoothers}. However, we did not exploit this fact in our implementation.

\subsubsection{Matrices based on a Gaussian kernel}
\label{sec:app:experiments:eigenvalues:gaussian}

For the \emph{spiral} dataset from Section~\ref{sec:app:experiments:spiral}, each data point comes with a feature vector $x_i \in \mathbb{R}^d$ and the adjacency matrix is given by 
\begin{equation}
\label{eq:app:gaussianadjacency}
W_{ij} = \begin{cases} \exp\left(\frac{-\|x_i-x_j\|^2}{\sigma^2}\right) & \text{if $i\neq j$,} \\ 0 & \text{else,} \end{cases}
\end{equation}
with a shape parameter $\sigma \in \mathbb{R}$.
Due to this special structure, we can use the method proposed in \cite{alfke18}. Due to the special structure of the adjacency matrix, matrix-vector products with $\mathcal{L}$ can be approximated efficiently and with a high degree of accuracy using a Non-equispaced Fast Fourier Transform (NFFT) \cite{alfke18}. This implementation can then be used in a Krylov-based eigenvalue algorithm like above, since these methods do not require the full matrix. For our experiments, we precompute the eigenvalues using the \texttt{fastsumAdjacencyEigs} function from the \textsc{Matlab} code\footnote{\url{https://www.tu-chemnitz.de/mathematik/wire/people/files_alfke/NFFT-Lanczos-Example-v1.tar.gz}} published with \cite{alfke18}. We summarize the arguments in Table~\ref{tbl:app:nfftarguments}. Default values were used for options not listed. The resulting computation time was 0.43 seconds.

\begin{table}
	\caption{Arguments for eigenvalue computation with the \texttt{fastsumAdjacencyEigs} function}
	\label{tbl:app:nfftarguments}
	\centering
	\begin{tabular}{@{}llp{0.6\textwidth}@{}}
		\toprule
		Parameter name & Value & Description \\
		\midrule
		\texttt{data} & $X$ & The feature matrix of size $n \times 3$ holding the three-dimensional data points $x_i$ in its rows. This is the same matrix used for the network input. See Section~\ref{sec:app:experiments:spiral} for details. \\[0.2em]
		\texttt{nev} & 11 & Number of eigenvalues computed. Since we use rank-10 filters, we need at most the smallest 11 eigenvalues of $\mathcal{L}$ as described in Section~\ref{sec:app:experiments:eigenvalues:which}. \\[0.2em]
		\texttt{opts.sigma} & $3.5$ & Shape parameter of the Gaussian kernel as in \eqref{eq:app:gaussianadjacency}, which must be chosen with respect to the dataset. We follow \cite{alfke18} in this choice. \\[0.2em]
		\texttt{opts.diagonalEntry} & 0 & Value for the diagonal of the adjacency matrix. Setting this option to 1 yields identity smoothing as in Section~\ref{sec:app:gcn:smooth}. \\[0.2em]
		\verb|opts.eigs_tol| & $10^{-3}$ & Target accuracy of the eigenvalue computation with \texttt{eigs}. Inaccurate results are acceptable due to the approximative nature of the NFFT method, following \cite{alfke18}. \\[0.2em]
		\texttt{opts.N} & 32 & \multirow{4}{*}{Technical parameters of the NFFT method, following \cite{alfke18}.} \\
		\texttt{opts.m} & 4 & \\
		\texttt{opts.p} & 1 & \\
		\verb|opts.eps_B| & 0 & \\
		\bottomrule
	\end{tabular}
\end{table}

Alternatively, eigenvalues can of course be computed as in Section~\ref{sec:app:experiments:eigenvalues:nostructure} by setting up the full matrix. Testing that approach in \textsc{Julia} yielded setup times of approximately 4.3 seconds.

\subsubsection{Hypergraph Laplacian in case of few hyperedges}
\label{sec:app:experiments:eigenvalues:hypergraph}

As described in Section~\ref{sec:app:hypergraph:efficient}, the eigenvalues of the hypergraph Laplacian (i.e., with hypergraph smoothing) can be computed via the nonzero singular values of a matrix $\tilde{H} \in \mathbb{R}^{n \times |E|}$. 

For full-rank filters as in Section~\ref{sec:app:hypergraph:efficient:fullrank}, we use the \texttt{svd} function from \text{Julia}'s standard \texttt{LinearAlgebra} module, which calls \textsc{Lapack}'s\footnote{\url{http://www.netlib.org/lapack/}} \texttt{gesdd!} function.

For low-rank filters as in Section~\ref{sec:app:hypergraph:efficient:lowrank}, we compute the $r$ or $r+1$ largest singular values of $\tilde{H}$ using the \texttt{svds} function from the \texttt{Arpack} package mentioned above, which in turn computes the $r$ largest eigenvalues of $\tilde{H}^T \tilde{H}$ with its \texttt{eigs} function.

\subsection{Training process}
\label{sec:app:experiments:training}

There are several details to be observed when specifying our training process.
\begin{itemize}
	
	\item \textbf{Run count:}
	All results published in our paper are averages over a number of \emph{runs}. For almost all architectures, the run count was 100. The only exceptions are the full-rank architectures with the naive implementation due to their training times of 15 to 40 minutes, where we only used 20 runs.
	
	\item \textbf{\textsc{TensorFlow} model:}
	For each run, a new network of tensors and operations was set up, representing the GCN architecture using the \textsc{TensorFlow.jl} module in \textsc{Julia}. Note that it would also have been possible to build the model once before the first one and then use new sessions for each run.

	\item \textbf{Weight initialization:}
	At the start of each run, the weight variables $\Theta^{(l)} \in \mathbb{R}^{N_{l-1}\times N_l}$ ($l=1,2$) are initialized randomly. All matrix entries are determined independently from a uniform distribution on the interval
	\[ \left[-\sqrt{\frac{6}{N_{l-1}+N_l}},\; \sqrt{\frac{6}{N_{l-1}+N_l}}\ \right]. \]
	This strategy was introduced as \emph{normalized initialization} in \cite{glorot10} and is popularly known as the uniform Glorot initializer.
	In our implementation, we run this variable assignment in our main \textsc{Julia} code instead of using \textsc{TensorFlow}'s initializer methods. This is done in order to have the randomness produced by \textsc{Julia}'s random number generator, for which we provide a fixed seed before each experiment.
	
	\item \textbf{Dropout:} Contrary to, e.g., \cite{kipf17}, we do not use any dropout because it did not improve results for us.
	
	\item \textbf{Training sets:}
	For each dataset, we used the same set of training nodes for all runs of all experiments. The same procedure is also used, e.g., in \cite{kipf17} for the popular citation networks datasets, which come with a predetermined training set and label rate. Since our datasets did not include a training set, we produced one ourselves as follows. First, we ran experiments with a new random training set for each run, which produced an expectedly high accuracy variance. We then fixed a training set containing the same number of nodes from each class and we tested whether it roughly produced the same average results as we had noticed before. We included this step to make sure that we did not obtain any exaggerated results. However, note that the ``quality'' of the training data does not favor any particular architecture, and our paper only contains comparisons of results produced with the same training set. Also, in real-world applications, the training data is rarely generated randomly, but often created manually with the goal to obtain the best results.
	
	The training sets used for each datasets are given in Sections~\ref{sec:app:experiments:spiral}, \ref{sec:app:experiments:cars}, and \ref{sec:app:experiments:mushrooms}, respectively, and they are also included in the datasets published with our experiment codes.
	
	\item \textbf{Loss function:}
	The weight variables $\Theta^{(l)}$ are trained to minimize a loss function. We follow the standard approach for classification problems  and use the cross entropy function, is the average negative logarithm of the computed probability for a training node to belong to its correct class. We further add a small regularization term punishing large entries in the first layer weight $\Theta^{(1)}$ via the squared Frobenius norm. Let $c_i \in \{1,\ldots,m\}$ denote the correct class of the $i$-th node ($i=1,\ldots,n$) and $\mathcal{I} \subset \{1,\ldots,n\}$ denote the training set. Then the full loss function is given by
	\[
	L(\Theta^{(1)}, \Theta^{(2)}) = -\frac{1}{|\mathcal{I}|} \sum_{i \in \mathcal{I}} \log \left( Y_{i,c_i} \right) + \frac{\rho}{2} \|\Theta^{(1)}\|_{\mathrm{F}}^2,
	\]
	where $Y$ depends on the weights in terms of \eqref{eq:app:architecture} or \eqref{eq:app:reducedarchitecture}.
	For the regularization factor $\rho$, we use the value $0.0005$ in all experiments. This loss function is equal to the one used in \cite{kipf17}.
	
	In \textsc{TensorFlow}, the combined softmax \eqref{eq:app:experiments:softmaxprediction} and cross entropy function can be computed via the \verb|softmax_cross_entropy_with_logits| method, and the half squared Frobenius norm is implemented in the \verb|l2_loss| method.
	
	\item \textbf{Optimization method:}
	In each run, we use a fixed number of 1000 iterations of the gradient descent method with a fixed learning rate (step length) of $0.2$. This choice is contrary to \cite{kipf17}, where a maximum of 200 iterations of the Adam optimizer from \cite{adam} were used with a learning rate of 0.01 and an additional early stopping strategy. However, we obtained better results with the gradient descent scheme. Also note that Adam was introduced as a method for stochastic optimization, while our loss function does not have any random elements. 
	This is due to the lack of dropout as well as the SSL setting, where the input is always the full matrix $X$ and no training batches are used.
	
	\item \textbf{Accuracy evaluation:}
	After each run, the number of correctly classified non-training nodes is determined. Node $i$ counts as correctly classified if the $i$-th row of $Y$ (or equivalently $X^{(2)}$) has its largest value in the $c_i$-th column, where $c_i \in \{1,\ldots,m\}$ is the correct class of that node. Nodes from the training set are neglected. The ratio of correctly classified nodes among all non-training nodes is reported as the accuracy of that run. The average accuracy of all runs is reported in the results tables of our paper.
	
	\item \textbf{Runtime partitioning:}
	For timing, we identity two main steps in the training process. The \emph{setup time} measures the computation of the system matrices for the layer operations, which includes any eigenvalue computation. The \emph{training time} measures all iterations of the optimization scheme. Not included are the initial load of the \textsc{TensorFlow} library and the dataset, the setup of the \textsc{TensorFlow} model, and the final evaluation.
	
\end{itemize}

\subsection{Spiral dataset}
\label{sec:app:experiments:spiral}

For the spiral dataset, we follow the data generation strategy used in \cite{alfke18}. First, the \textsc{Matlab} function \texttt{generateSpiralDataWithLabels.m}\footnote{The link from \cite{alfke18} is no longer available, but the file was also published, e.g., at \url{http://www.codelooker.com/id/180/1117266.html} as part of experiments with the \textsc{Matlab} interface of the \texttt{libsvm} library for support vector machines.} is called with the parameters listed in Table~\ref{tbl:app:spiralgeneration}. The code generates five three-dimensional orb centers based on the parameters. For each center, 2000 points are generated following a multivariate normal distribution with expected value set to the orb center coordinates. For the covariance matrix, the identity is used. Originally, the true label of each data point was set to the index of the orb center for which it was generated. However, since the orbs intersect, this may lead to nodes that are impossible to classify correctly. Thus we follow \cite{alfke18} in setting the true label of each point to the index of the orb center closest to it in the Euclidean norm. A description of the dataset is given in Table~\ref{tbl:app:spiraldescription}. A smaller version of the dataset with only 400 points per orb is visualized in Figure~\ref{fig:app:spiraldataset}, similar to \cite[Figure 2a]{alfke18}.

\begin{table}
	\caption{Parameters for spiral dataset generation}
	\label{tbl:app:spiralgeneration}
	\centering
	\begin{tabular}{@{}llp{0.6\textwidth}@{}}
		\toprule
		Parameter name & Value & Description \\
		\midrule
		\texttt{Nc} & 5 & Number of orbs and classes \\
		\texttt{Ns} & 2000 & Number of data points per orb \\
		\texttt{h} & 10 & The $z$-coordinates of the five orb centers lie equidistantly between 0 and \texttt{h} \\
		\texttt{r} & 2 & The $x$-$y$-coordinates of the five orb centers lie on a circle with radius \texttt{r} \\
		\bottomrule & 
	\end{tabular}
\end{table}

\begin{table}
	\caption{Description of the spiral dataset}
	\label{tbl:app:spiraldescription}
	\centering
	\begin{tabular}{@{}cS[table-format=3.2]S[table-format=3.2]S[table-format=3.2]ccc@{}}
		\toprule
		Orb & \multicolumn{3}{c}{Center coordinates} & Number of points & \multicolumn{2}{c}{Indices of} \\
		index & $x$ & $y$ & $z$ & with that label & \multicolumn{2}{c}{training nodes} \\
		\midrule
		1 & 2.0 & 0.0 & 0.0 & 1972 & 1736 & 1869 \\
		2 & 0.6180 & 1.9021 & 2.5 & 2026 & 2949 & 3785 \\
		3 & -1.6180 & 1.1756 & 5.0 & 1999 & 4187 & 4206 \\
		4 & -1.6180 & -1.1756 & 7.5 & 1981 & 6532 & 7186  \\
		5 & 0.6180 & -1.9021 & 10.0 & 2022 & 8296 & 9961 \\
		\bottomrule
	\end{tabular}
\end{table}

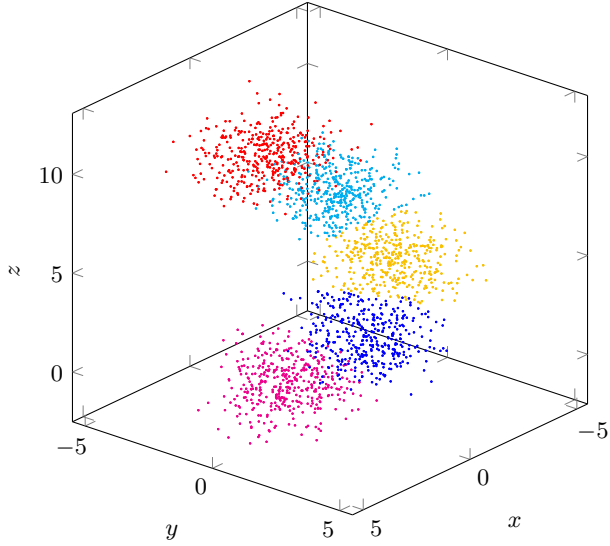
\begin{figure}
	\centering
	\begin{tikzpicture}[baseline]
	\begin{axis}[
	view={130}{25},
	enlargelimits=false,
	font=\footnotesize,
	width=0.6\textwidth, height=0.6\textwidth,
scatter/classes={
		1={mark=*,mark size=0.25,magenta,style={solid, fill=magenta}},
		2={mark=*,mark size=0.25,blue,style={solid, fill=blue}},
		3={mark=*,mark size=0.25,draw=amber,style={solid, fill=green}},
		4={mark=*,mark size=0.25,draw=cyan,style={solid, fill=cyan}},
		5={mark=*,mark size=0.25,draw=red,style={solid, fill=red}}
	},
	xlabel={$x$}, ylabel={$y$}, zlabel={$z$},
	xmin=-5.5, xmax=5.5,
	ymin=-5.5, ymax=5.5,
	]
	\addplot3[scatter,only marks,scatter src=explicit symbolic] file {spiral_2000.txt};
	\end{axis} 
	\end{tikzpicture}
	\caption{Example of the spiral dataset with $n=2000$ points}
	\label{fig:app:spiraldataset}
\end{figure}

The layer widths are $N_0=3$, $N_1=4$, and $N_2=5$, leading to $N_0 N_1 = 12$ weights in the first layer matrix and $N_1 N_2 = 20$ weights in the second layer matrix.

Eigenvalues are computed as described in Section~\ref{sec:app:experiments:eigenvalues:gaussian}, which took approximately 0.39 seconds for the first 10 eigenvalues, or 0.44 seconds for the first 11 eigenvalues.

\subsection{Cars dataset}
\label{sec:app:experiments:cars}

The cars dataset is based on six general properties of cars. Three of these (buying price, maintenance price, and number of doors) have four possible discrete values each, while the other three properties (number of persons it can carry, luggage boot size, and safety) have three possible discrete values each. The data points are the $n=4\cdot4\cdot4\cdot3\cdot3\cdot3=1728$ possible property value combinations. For each of these, one of $m=4$ acceptability classes was assigned manually. As described in our paper, a hypergraph is constructed from these properties by creating hyperedges containing all nodes for which a certain property has a certain value, e.g., all car designs with exactly four doors, or with a small luggage boot. This way, $|E|=4+4+4+3+3+3=21$ hyperedges are created, each with an assigned weight of $w_e=1$. The incidence matrix $H$ of this hypergraph is used as the GCN input matrix $X$.

Due to the symmetric nature of the dataset, all hyperedges connect either $n/3=576$ or $n/4=432$ nodes. Each node is included in exactly six hyperedges, three of which have $|e|=576$ and $|e|=432$, respectively. As a result, the node degree matrix is $D_H = 6 I$ and the loop weight matrix 
\eqref{eq:app:hypergraph:loopweights}
is $S_H = s I$ with $s = \frac{3}{576}+\frac{3}{432} \approx 0.01215$. Hence hypergraph smoothing is identical to identity smoothing with a very small factor and even any linear combinations of those, like the re-normalization trick for hypergraphs \eqref{eq:app:hypergraph:combinedsmoothing}, can just be written as $S=\alpha I$.
We can hence simplify expression \eqref{eq:app:hypergraph:laplaciansplitting} to
\[
\mathcal{L}_{\alpha I} = \frac{1}{6 + \alpha - s} \left( 6 I - H D_E^{-1} H^T \right),
\]
where the second part is positive semidefinite with rank $|E| = 21$. The largest eigenvalue of $\mathcal{L}_{\alpha I}$ is always $\frac{6}{6+\alpha-s}$ with a multiplicity of $n-|E| = 1707$.

This shows that all of our smoothing approaches only lead to a scaled graph Laplacian, and that the efficient evaluation techniques from Section~\ref{sec:app:hypergraph:efficient} can also be used for differently smoothed Laplacians. Using a smoother is equivalent to changing coefficients in the filter function. That being said, it is still preferable to use the hypergraph Laplacian because our filter functions are designed for a largest eigenvalue of 1, which is achieved with $\alpha=s$.

The layer widths are $N_0=21$, $N_1=8$, and $N_2=4$, leading to $N_0 N_1 = 168$ weights in the first layer matrix and $N_1 N_2 = 32$ weights in the second layer matrix.

As described in Section~\ref{sec:app:experiments:training}, we used a fixed training set of 20 nodes with the indices given below. For each class, there are five training nodes.
\begin{align*}
26&& 227&& 232&& 251&& 310&& 501&& 648&& 688&& 718&& 863 \\
882&& 1073&& 1075&& 1167&& 1291&& 1388&& 1494&& 1554&& 1621&& 1705
\end{align*}

\subsection{Mushrooms dataset}
\label{sec:app:experiments:mushrooms}

The mushrooms dataset contains $n=8124$ sample descriptions of mushroom species, described by 22 discrete properties, of which we discard one due to missing values. For the other 21 properties, we apply the same procedure as for the cars dataset, producing $|E|=112$ hyperedges, all with a weight of $w_e=1$. The goal is to decide whether the mushrooms are edible or inedible, where the latter class contains poisonous mushrooms as well as those of unknown edibility and where eating is not recommended. The dataset was also used, e.g., in \cite{zhou06} and \cite{yadati19}.

The layer widths are $N_0=112$, $N_1=16$, and $N_2=2$, leading to $N_0 N_1 = 1792$ weights in the first layer matrix and $N_1 N_2 = 32$ weights in the second layer matrix.

As described in Section~\ref{sec:app:experiments:training}, we used a fixed training set of 20 nodes with the indices given below. There are exactly 10 edible and 10 inedible training samples.
\begin{align*}
224&& 610&& 939&& 1430&& 1743&& 2442&& 2559&& 3129&& 4268&& 4286 \\
4354&& 4713&& 5602&& 5615&& 5845&& 6434&& 6486&& 6744&& 7515&& 7954
\end{align*}

\subsection{Results for the mushrooms dataset with other ranks}
\label{sec:app:experiments:mushroomsrankcomparison}

In this section, we report the results of an investigation into the effect that the target rank $r$ has on the performance of low-rank kernels. We focus on the pseudoinverse filter function in the L-GCN and R-GCN architectures with the hypergraph Laplacian for the mushrooms dataset. All parameters are the same as in the experiments for our main paper, except for the kernel matrix rank. We test different ranks $r\in \{10,12,15,20,25,30,40,50,70,100\}$. Results are visualized in Figure~\ref{fig:app:rankcomparison}.

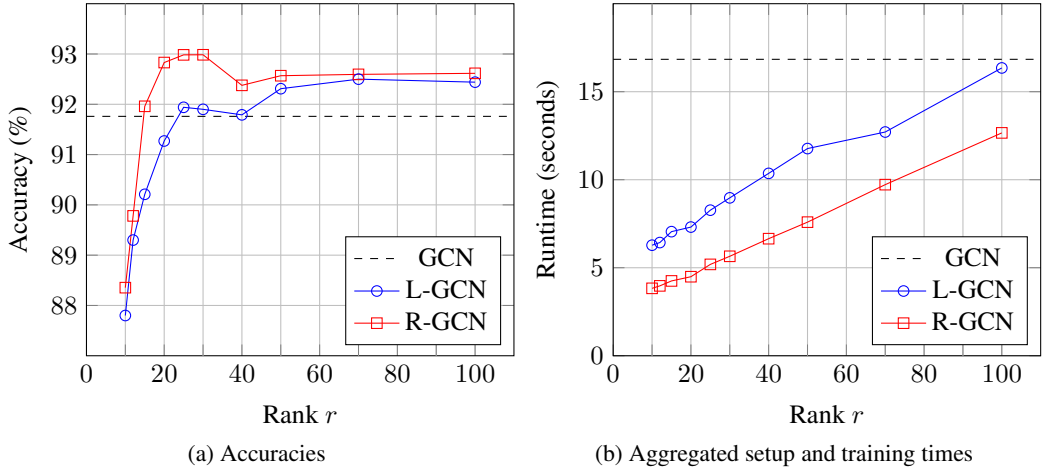
\begin{figure}
	\subfloat[Accuracies]{
		\begin{tikzpicture}
		\begin{axis}[
		width = 0.52\textwidth,
		xlabel={Rank $r$}, ylabel={Accuracy (\%)},
		xtick={0,20,...,100},
		extra x ticks = {10,30,...,90},
		extra x tick labels = {},
ytick={88,...,93},
grid=major,
		xmin=0, xmax=110,
ymin=87, ymax=94,
		ylabel near ticks,
		legend pos=south east,
		]
		\addplot[black,dashed,domain=0:110] {91.76};
		\addlegendentry{GCN}
		
		\addplot[blue, mark=o] coordinates {(10, 87.8) (12, 89.3) (15, 90.21) (20, 91.27) (25, 91.94) (30, 91.9) (40, 91.79) (50, 92.31) (70, 92.5) (100, 92.44)};
		\addlegendentry{L-GCN}
		
		\addplot[red, mark=square] coordinates {(10, 88.3539) (12, 89.7808) (15, 91.9600) (20, 92.83) (25, 92.9827) (30, 92.9844) (40, 92.3769) (50, 92.5682) (70, 92.5949) (100, 92.6157)};
		\addlegendentry{R-GCN}
		
		\end{axis}
		\end{tikzpicture}
	}
\hfill
	\subfloat[Aggregated setup and training times]{
		\begin{tikzpicture}
		\begin{axis}[
		width = 0.52\textwidth,
		xlabel={Rank $r$}, ylabel={Runtime (seconds)},
		xtick={0,20,...,100},
		extra x ticks = {10,30,...,90},
		extra x tick labels = {},
ytick={0,5,...,15},
		grid=major,
		xmin=0, xmax=110,
		ymin=0, ymax=20,
		ylabel near ticks,
		legend pos=south east,
		]
		\addplot[black,dashed,domain=0:110] {16.84};
		\addlegendentry{GCN}
		
		\addplot[blue, mark=o] coordinates {(10, 6.28) (12, 6.43) (15, 7.05) (20, 7.31) (25, 8.27) (30, 8.97) (40, 10.36) (50, 11.77) (70, 12.71) (100, 16.35)};
		\addlegendentry{L-GCN}
		
		\addplot[red, mark=square] coordinates {(10, 3.83) (12, 3.97) (15, 4.25) (20, 4.49) (25, 5.19) (30, 5.65) (40, 6.65) (50, 7.59) (70, 9.72) (100, 12.66)};
		\addlegendentry{R-GCN}
		
		\end{axis}
		\end{tikzpicture}
	}
	\caption{Performance development of the low-rank and reduced-order architectures over different ranks}
	\label{fig:app:rankcomparison}
\end{figure}

Even with a target rank of $10$, we gain an accuracy of around 88\% with both architectures, which is worse than the other results but might still be acceptable. Starting from a rank of $20$, accuracy results change by at most 0.5\%. This is most likely due to the fact that the eigenvectors to the 30 smallest eigenvalues contain the best clustering information. Adding more eigenvectors does not necessarily add new beneficial information, and increasing the rank above 30 even seems to be detrimental. The reduced-order architecture yields its best results around a rank of 25, while the low-rank architecture gains additional benefits from a rank of 70. The runtimes develop approximately in an affine linear fashion, as could be expected.

All in all, even though there is no clear best rank to use, a rank of 25 seems to be slightly better than the rank we used in our main paper, $r=20$. However, in real-world applications, an in-depth rank comparison is rarely suitable, so the target rank has to be chosen beforehand and this evaluation supports that $r=20$ is a reasonable choice.

\subsection{Results for the mushrooms dataset with other smoothers}
\label{sec:app:experiments:smoothedmushrooms}

We will now give a quick overview over the results obtained with different smoothing approaches, as described in Section~\ref{sec:app:hypergraph:arbitrarysmoothers}. We will restrict our experiments to the mushrooms dataset and two different architectures: the traditional GCN with the linear filter, evaluated as in Section~\ref{sec:app:hypergraph:efficient:linear}, and the reduced-order GCN with the pseudoinverse filter of rank 20. We compare four different smoothing strategies for the matrix $S$ in \eqref{eq:app:hypergraph:laplaciansplitting}:
\begin{itemize}
	\item \textbf{No smoothing} ($S=0$). This strategy means that we transform the hypergraph into a graph without loops, following Section~\ref{sec:app:hypergraph:interpretation}, and then use the traditional Laplacian of that graph.
	\item \textbf{Identity smoothing} ($S=I$). Adding the identity matrix as a smoother is equivalent to using the non-smoothed Laplacian and then employing the re-normalization trick in the layer operation, as seen in Section~\ref{sec:app:gcn:smooth}.
	\item \textbf{Hypergraph smoothing} ($S=S_H$). As described in Section~\ref{sec:app:hypergraph:interpretation}, choosing the smoother matrix from \eqref{eq:app:hypergraph:loopweights} retrieves the original hypergraph Laplacian. This is the only smoothing approach that allows for the efficient eigenvalue computation and convolution formulas given in Section~\ref{sec:app:hypergraph:efficient}, and it is the one used in our main paper.
	\item \textbf{Combined hypergraph and identity smoothing} ($S=S_H+I$). This strategy is equivalent to using the hypergraph Laplacian and additionally employing the re-normalization trick, as described in \eqref{eq:app:hypergraph:combinedsmoothing} and used in \cite{feng19}.
\end{itemize}

\begin{table}
	\caption{Results for the mushrooms dataset with R-GCN architecture, pseudoinverse filter, and different types of smoothing}
	\label{tbl:mushrooms:smoothercomparison}
	\centering
	\begin{tabular}{@{}llcccc@{}}
		\toprule
		\multirow{2}{*}{Network} & \multirow{2}{*}{Filter function} & \multirow{2}{*}{Smoother} & \multicolumn{2}{c}{Time} & \multirow{2}{*}{Accuracy} \\ \cmidrule(lr){4-5}
		&&& Setup & Training & \\
		\midrule
		\multirow{4}{4em}{GCN (efficient)} & \multirow{4}{*}{Linear} & None & \multirow{4}{*}{$<$0.01 s} & 17.51 s & 88.54 \% \\
		&& Identity &  & 17.54 s & 89.26 \% \\
		&& Hypergraph & & 16.73 s & 88.82 \% \\
		&& Combined &  & 17.57 s & \textbf{89.31} \% \\
		\midrule
		\multirow{4}{4em}{R-GCN (rank 20)} & \multirow{4}{*}{Pseudoinverse} & None & 2.31 s & 4.35 s & 91.85 \% \\
		&& Identity & 2.35 s & 4.39 s & 91.68 \% \\
		&& Hypergraph & 0.05 s & 4.44 s & \textbf{92.83} \% \\
		&& Combined & 1.80 s & 4.58 s & 91.72 \% \\
		\bottomrule
	\end{tabular}
\end{table}

Results are given in Table~\ref{tbl:mushrooms:smoothercomparison}. For the full-rank linear architecture, setup is always fast, and the re-normalization trick slightly improves accuracy, making the strategy from \cite{feng19} the best choice. For the reduced-order pseudoinverse architecture, the hypergraph Laplacian yields a drastically reduced setup time as well as the best accuracy, making it an obvious choice for this architecture. In different scenarios, different smoothers might yield the best accuracy, but hypergraph smoothing always gains faster setup for all but the full-rank linear filters.

\bibliographystyle{abbrvnat}
\bibliography{lowrankgcn-refs}

\begin{thebibliography}{24}
\providecommand{\natexlab}[1]{#1}
\providecommand{\url}[1]{\texttt{#1}}
\expandafter\ifx\csname urlstyle\endcsname\relax
  \providecommand{\doi}[1]{doi: #1}\else
  \providecommand{\doi}{doi: \begingroup \urlstyle{rm}\Url}\fi

\bibitem[Alfke et~al.(2018)Alfke, Potts, Stoll, and Volkmer]{alfke18}
D.~Alfke, D.~Potts, M.~Stoll, and T.~Volkmer.
\newblock {NFFT} meets {K}rylov methods: Fast matrix-vector products for the
  graph {L}aplacian of fully connected networks.
\newblock \emph{Frontiers in Applied Mathematics and Statistics}, 4, 2018.

\bibitem[Antoulas(2005)]{antoulas05}
A.~C. Antoulas.
\newblock \emph{Approximation of Large-Scale Dynamical Systems}.
\newblock Advances in Design and Control. {SIAM}, 2005.

\bibitem[Bosch et~al.(2016)Bosch, Klamt, and Stoll]{bosch16}
J.~Bosch, S.~Klamt, and M.~Stoll.
\newblock Generalizing diffuse interface methods on graphs: Nonsmooth
  potentials and hypergraphs.
\newblock \emph{SIAM Journal on Applied Mathematics}, 78, 2016.

\bibitem[Bretto(2013)]{bretto13}
A.~Bretto.
\newblock \emph{Hypergraph Theory}.
\newblock Mathematical Engineering. Springer, 2013.

\bibitem[Bruna et~al.(2014)Bruna, Zaremba, Szlam, and LeCun]{bruna14}
J.~Bruna, W.~Zaremba, A.~Szlam, and Y.~LeCun.
\newblock Spectral networks and locally connected networks on graphs.
\newblock In \emph{Proceedings of the 2nd International Conference on Learning
  Representations}, 2014.

\bibitem[Defferrard et~al.(2016)Defferrard, Bresson, and
  Vandergheynst]{defferard16}
M.~Defferrard, X.~Bresson, and P.~Vandergheynst.
\newblock Convolutional neural networks on graphs with fast localized spectral
  filtering.
\newblock In \emph{Advances in Neural Information Processing Systems 29}, 2016.

\bibitem[Feng et~al.(2019)Feng, You, Zhang, Ji, and Gao]{feng19}
Y.~Feng, H.~You, Z.~Zhang, R.~Ji, and Y.~Gao.
\newblock Hypergraph neural networks.
\newblock Preprint available at \url{https://arxiv.org/abs/1809.09401}, 2019.

\bibitem[Gao et~al.(2013)Gao, Wang, Zha, Shen, Li, and Wu]{gao13}
Y.~Gao, M.~Wang, Z.-J. Zha, J.~Shen, X.~Li, and X.~Wu.
\newblock Visual-textual joint relevance learning for tag-based social image
  search.
\newblock \emph{{IEEE} Transactions on Image Processing}, 22, 2013.

\bibitem[Glorot and Bengio(2010)]{glorot10}
X.~Glorot and Y.~Bengio.
\newblock Understanding the difficulty of training deep feedforward neural
  networks.
\newblock In \emph{Proceedings of the Thirteenth International Conference on
  Artificial Intelligence and Statistics}, 2010.

\bibitem[Golub and Van~Loan(1996)]{golubvanloan}
G.~H. Golub and C.~F. Van~Loan.
\newblock \emph{Matrix Computations}.
\newblock The Johns Hopkins University Press, 3rd edition, 1996.

\bibitem[Hein et~al.(2013)Hein, Setzer, Jost, and Rangapuram]{hein13}
M.~Hein, S.~Setzer, L.~Jost, and S.~S. Rangapuram.
\newblock The {T}otal {V}ariation on hypergraphs -- {L}earning on hypergraphs
  revisited.
\newblock In \emph{Advances in Neural Information Processing Systems 26}, 2013.

\bibitem[Horn and Johnson(1985)]{hornjohnson85}
R.~A. Horn and C.~R. Johnson.
\newblock \emph{Matrix Analysis}.
\newblock Cambridge University Press, 1985.

\bibitem[Kingma and Ba(2015)]{adam}
D.~Kingma and J.~L. Ba.
\newblock Adam: A method for stochastic optimization.
\newblock In \emph{Proceedings of the 3rd International Conference on Learning
  Representations}, 2015.

\bibitem[Kipf and Welling(2017)]{kipf17}
T.~N. Kipf and M.~Welling.
\newblock Semi-supervised classification with graph convolutional networks.
\newblock In \emph{Proceedings of the 5th International Conference on Learning
  Representations}, 2017.

\bibitem[Langone et~al.(2016)Langone, Van~Barel, and Suykens]{langone16}
R.~Langone, M.~Van~Barel, and J.~A. Suykens.
\newblock Efficient evolutionary spectral clustering.
\newblock \emph{Pattern Recognition Letters}, 84, 2016.

\bibitem[Lehoucq and Sorensen(1996)]{lehoucq96}
R.~B. Lehoucq and D.~C. Sorensen.
\newblock Deflation techniques for an implicitly restarted arnoldi iteration.
\newblock \emph{{SIAM} Journal on Matrix Analysis and Applications}, 17, 1996.

\bibitem[Lehoucq et~al.(1998)Lehoucq, Sorensen, and Yang]{arpack98}
R.~B. Lehoucq, D.~C. Sorensen, and C.~Yang.
\newblock \emph{{ARPACK} Users' Guide: Solution of Large-Scale Eigenvalue
  Problems with Implicitly Restarted {A}rnoldi Methods}.
\newblock Software, Environments, and Tools. {SIAM}, 1998.

\bibitem[Li et~al.(2018)Li, Han, and Wu]{li18}
Q.~Li, Z.~Han, and X.~Wu.
\newblock Deeper insights into graph convolutional networks for semi-supervised
  learning.
\newblock In \emph{The Thirty-Second AAAI Conference on Artificial
  Intelligence}, 2018.

\bibitem[Purkait et~al.(2017)Purkait, Chin, Sadri, and Suter]{purkait17}
P.~Purkait, T.-J. Chin, A.~Sadri, and D.~Suter.
\newblock Clustering with hypergraphs: The case for large hyperedges.
\newblock \emph{IEEE Transactions on Pattern Analysis and Machine
  Intelligence}, 39, 2017.

\bibitem[{Shuman} et~al.(2013){Shuman}, {Narang}, {Frossard}, {Ortega}, and
  {Vandergheynst}]{shuman13}
D.~{Shuman}, S.~{Narang}, P.~{Frossard}, A.~{Ortega}, and P.~{Vandergheynst}.
\newblock The emerging field of signal processing on graphs: Extending
  high-dimensional data analysis to networks and other irregular domains.
\newblock \emph{IEEE Signal Processing Magazine}, 30, 2013.

\bibitem[{von Luxburg}(2007)]{luxburg07}
U.~{von Luxburg}.
\newblock A tutorial on spectral clustering.
\newblock \emph{Statistics and Computing}, 17, 2007.

\bibitem[{Wu} et~al.(2019){Wu}, {Pan}, {Chen}, {Long}, {Zhang}, and {Yu}]{wu19}
Z.~{Wu}, S.~{Pan}, F.~{Chen}, G.~{Long}, C.~{Zhang}, and P.~S. {Yu}.
\newblock A comprehensive survey on graph neural networks.
\newblock Preprint available at \url{https://arxiv.org/abs/1901.00596}, 2019.

\bibitem[Yadati et~al.(2019)Yadati, Nimishakavi, Yadav, Louis, and
  Talukdar]{yadati19}
N.~Yadati, M.~Nimishakavi, P.~Yadav, A.~Louis, and P.~Talukdar.
\newblock {HyperGCN}: Hypergraph convolutional networks for semi-supervised
  classification.
\newblock Preprint available at \url{https://arxiv.org/abs/1809.02589}, 2019.

\bibitem[Zhou et~al.(2006)Zhou, Huang, and Sch\"olkopf]{zhou06}
D.~Zhou, J.~Huang, and B.~Sch\"olkopf.
\newblock Learning with hypergraphs: clustering, classification, and embedding.
\newblock In \emph{Advances in Neural Information Processing Systems 19}, 2006.

\end{thebibliography}

\end{document}